\documentclass[sigconf]{acmart}

\usepackage{graphicx}
\usepackage{wrapfig}
\usepackage{multirow}
\usepackage{color, colortbl}
\usepackage[caption=false,font=normalsize,labelfont=sf,textfont=sf,position=t]{subfig}
\usepackage{appendix}
\usepackage{caption}

\usepackage{amsmath,amsfonts,bm}

\def\eqref#1{equation~\ref{#1}}

\def\1{\bm{1}}

\def\vh{{\mathbf{h}}}

\def\vx{{\mathbf{x}}}
\def\vy{{\mathbf{y}}}

\def\mA{{\mathbf{A}}}

\def\mD{{\mathbf{D}}}

\def\mF{{\mathbf{F}}}

\def\mH{{\mathbf{H}}}
\def\mI{{\mathbf{I}}}

\def\mM{{\mathbf{M}}}

\def\mT{{\mathbf{T}}}

\def\mX{{\mathbf{X}}}
\def\mY{{\mathbf{Y}}}
\def\mZ{{\mathbf{Z}}}

\DeclareMathAlphabet{\mathsfit}{\encodingdefault}{\sfdefault}{m}{sl}
\SetMathAlphabet{\mathsfit}{bold}{\encodingdefault}{\sfdefault}{bx}{n}

\def\gB{{\mathcal{B}}}

\def\gD{{\mathcal{D}}}
\def\gE{{\mathcal{E}}}

\def\gG{{\mathcal{G}}}

\def\gI{{\mathcal{I}}}

\def\gL{{\mathcal{L}}}

\def\gN{{\mathcal{N}}}

\def\gS{{\mathcal{S}}}

\def\gV{{\mathcal{V}}}

\newcommand{\E}{\mathbb{E}}

\newcommand{\R}{\mathbb{R}}

\newcommand{\Cov}{\mathrm{Cov}}

\DeclareMathOperator*{\argmin}{arg\,min}

\theoremstyle{plain}
\newtheorem{theorem}{Theorem}[section]
\newtheorem{proposition}[theorem]{Proposition}
\newtheorem{lemma}[theorem]{Lemma}

\theoremstyle{definition}

\theoremstyle{remark}

\definecolor{Gray}{gray}{0.95}
\definecolor{darkgreen}{RGB}{0, 100, 0}
\definecolor{darkred}{RGB}{139, 0, 0}
\definecolor{First}{HTML}{BFC0FF} 
\definecolor{Second}{HTML}{E7E6FF} 

\AtBeginDocument{%
  }

\copyrightyear{2025}
\acmYear{2025}
\setcopyright{acmlicensed}
\acmConference[WSDM '25]{Proceedings of the Eighteenth ACM International Conference on Web Search and Data Mining}{March 10--14, 2025}{Hannover, Germany}
\acmBooktitle{Proceedings of the Eighteenth ACM International Conference on Web Search and Data Mining (WSDM '25), March 10--14, 2025, Hannover, Germany}
\acmDOI{10.1145/3701551.3703550}
\acmISBN{979-8-4007-1329-3/25/03}

\begin{document}

\title{Training MLPs on Graphs without Supervision}


\author{Zehong Wang}
\affiliation{%
  \institution{University of Notre Dame}
  \city{Notre Dame}
  \state{Indiana}
  \country{USA}}
\email{zwang43@nd.edu}

\author{Zheyuan Zhang}
\affiliation{%
  \institution{University of Notre Dame}
  \city{Notre Dame}
  \state{Indiana}
  \country{USA}}
\email{zzhang42@nd.edu}

\author{Chuxu Zhang}
\affiliation{%
  \institution{University of Connecticut}
  \city{Storrs}
  \state{Connecticut}
  \country{USA}}
\email{chuxu.zhang@uconn.edu}

\author{Yanfang Ye}
\authornote{Corresponding Author.}
\affiliation{%
  \institution{University of Notre Dame}
  \city{Notre Dame}
  \state{Indiana}
  \country{USA}}
\email{yye7@nd.edu}

\begin{abstract}
  Graph Neural Networks (GNNs) have demonstrated their effectiveness in various graph learning tasks, yet their reliance on neighborhood aggregation during inference poses challenges for deployment in latency-sensitive applications, such as real-time financial fraud detection. To address this limitation, recent studies have proposed distilling knowledge from teacher GNNs into student Multi-Layer Perceptrons (MLPs) trained on node content, aiming to accelerate inference. However, these approaches often inadequately explore structural information when inferring unseen nodes. To this end, we introduce \textbf{SimMLP}, a \textbf{S}elf-superv\textbf{i}sed fra\textbf{m}ework for learning \textbf{MLP}s on graphs, designed to fully integrate rich structural information into MLPs. Notably, SimMLP is the first MLP-learning method that can achieve equivalence to GNNs in the optimal case. The key idea is to employ self-supervised learning to align the representations encoded by graph context-aware GNNs and neighborhood dependency-free MLPs, thereby fully integrating the structural information into MLPs. We provide a comprehensive theoretical analysis, demonstrating the equivalence between SimMLP and GNNs based on mutual information and inductive bias, highlighting SimMLP's advanced structural learning capabilities. Additionally, we conduct extensive experiments on 20 benchmark datasets, covering node classification, link prediction, and graph classification, to showcase SimMLP's superiority over state-of-the-art baselines, particularly in scenarios involving unseen nodes (e.g., inductive and cold-start node classification) where structural insights are crucial. Our codes are available at: \url{https://github.com/Zehong-Wang/SimMLP}.

\end{abstract}

\ccsdesc[500]{Information systems~Data mining}
\ccsdesc[500]{Computing methodologies~Unsupervised learning}
\ccsdesc[500]{Computing methodologies~Neural networks}

\keywords{Graph Neural Networks; Inference Acceleration; Self-Supervision}


\maketitle

\section{Introduction}

\begin{figure}[!t]
  \centering
  \includegraphics[width=0.9\linewidth]{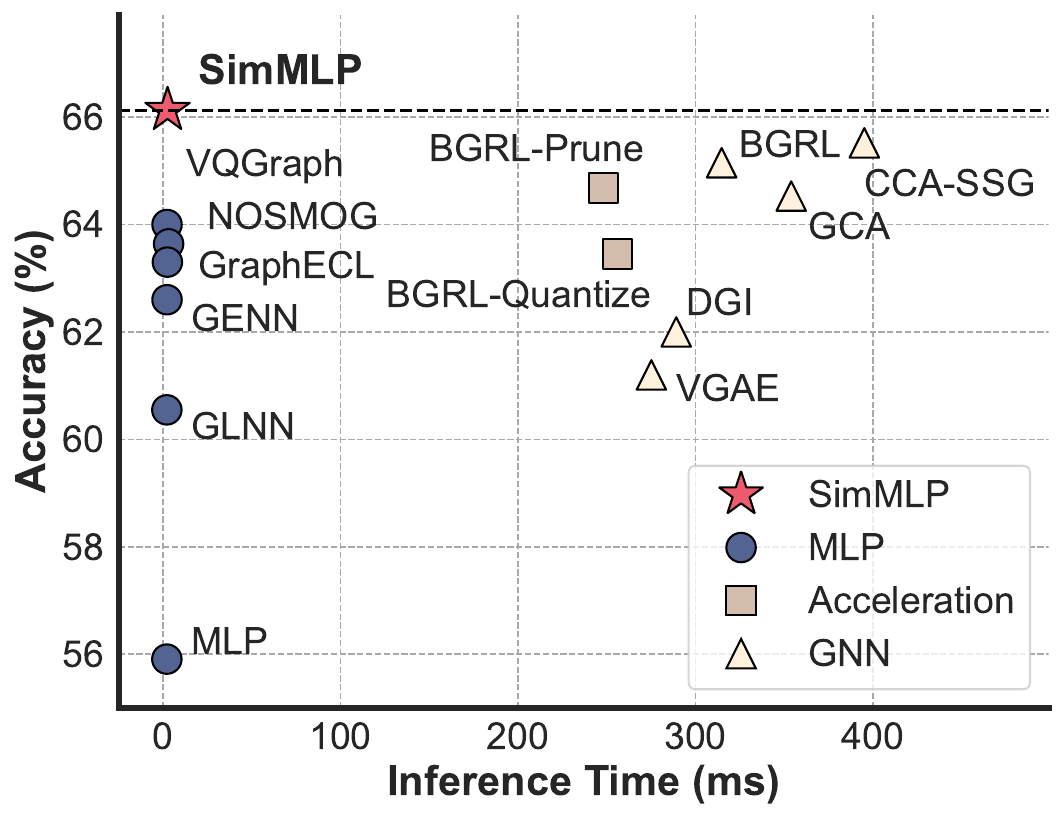}
  \vspace{-5pt}
  \caption{Accuracy vs. Inference Time on Arxiv dataset under cold-start setting. }
  \vspace{-5pt}
  \label{fig:time_acc}
\end{figure}

Given the widespread presence of graph-structured data, such as computing networks, e-commerce recommender systems, citation networks, and social networks, Graph Neural Networks (GNNs) have drawn significant attention in recent years. Typically, GNNs rely on message-passing mechanisms \citep{gilmer2017neural} to iteratively capture neighborhood information and learn representations of graphs. Despite their effectiveness across various graph learning tasks, GNNs face challenges in latency-sensitive applications, such as financial fraud detection \citep{wang2021apan}, due to the computational overhead associated with neighborhood fetching \citep{zhang2022graphless}. To enable faster inference in downstream tasks, existing approaches primarily employ techniques like quantization \citep{ding2021vq}, pruning \citep{zhou2021accelerating}, and knowledge distillation \citep{yan2020tinygnn} for accelerating graph inference. However, these methods are still constrained by the need to fetch neighborhoods, limiting their effectiveness in real-time scenarios.

To address this issue, Multi-Layer Perceptrons (MLPs), trained solely on node features without message passing, have emerged as efficient alternatives for latency-sensitive applications, offering up to two orders of magnitude acceleration ($100\times$) compared to GNNs \citep{zhang2022graphless,guo2023linkless,tian2022learning}. However, despite this significant speedup, the absence of message passing inevitably hinders the ability to capture structural information, resulting in degraded model performance. To mitigate this, \citet{zhang2022graphless} proposed distilling knowledge from pre-trained GNN teachers into MLP students by minimizing the KL divergence between the predictions of GNNs and MLPs. Although this approach has notably improved MLP performance, it has been observed that the model tends to mimic GNN predictions using MLPs rather than truly understanding the localized structural information of nodes, leading to sub-optimal performance.

To address this limitation, researchers intend to incorporate structural knowledge into MLPs. For instance, \citet{tian2022learning,wang2023graph} employ random walks on the original graph to learn positional embeddings for each node, which are then appended to the raw node features as complementary information. \citet{yang2023vqgraph} use vector quantization \citep{van2017neural} to learn a structural codebook during GNN pre-training and subsequently distill the knowledge from the codebook into downstream MLPs. Additionally, \citet{hu2021graph,xiao2024efficient} leverage contrastive learning to encode localized structures by pulling the target node closer to its neighbors in the embedding space. In summary, these methods utilize different heuristics--positional embeddings \citep{tian2022learning,wang2023graph}, structural codebooks \citep{yang2023vqgraph}, and neighborhood relationships \citep{hu2021graph,xiao2024efficient}--to provide structural information during MLP training. However, these heuristic approaches \textit{cannot fully replicate the functionality of GNNs in capturing the complete structural information of graphs}, which may lead to reduced generalization to unseen nodes.

To this end, we present a simple yet effective method, \textbf{SimMLP}: a \textbf{S}elf-superv\textbf{i}sed fra\textbf{m}ework for learning \textbf{MLP}s on graphs, which is the first MLP learning method equivalent to GNNs in the optimal case. SimMLP is based on the insight that modeling the fine-grained correlation between node features and graph structures can enhance the generalization of node embeddings \citep{tian2020makes}. Building on this insight, we employ a self-supervised loss to maximize the alignment between GNNs and MLPs in the embedding space, preserving intricate semantic knowledge. Theoretically, we demonstrate the equivalence between SimMLP and GNNs based on mutual information maximization and inductive biases. Furthermore, we interpret SimMLP through the lens of the information bottleneck theory to illustrate its generalization capability. This equivalence to GNNs offers three distinct advantages compared to existing methods:
(1) \textbf{Generalization}: SimMLP has the capability to fully comprehend the localized structures of nodes, making it well-suited for scenarios involving unseen nodes, such as inductive \citep{hamilton2017inductive} or cold-start settings \citep{zheng2022cold}. (2) \textbf{Robustness}: SimMLP demonstrates robustness to both feature and edge noise due to its advanced structural utilization. Additionally, the self-supervised nature mitigates the risk of overfitting to specific structural patterns, enhancing its resilience in situations with label scarcity. (3) \textbf{Versatility}: The self-supervised alignment in SimMLP enables the acquisition of task-agnostic knowledge, allowing the model to be applied across various graph-related tasks, whereas existing methods are tailored for a single task. To evaluate these benefits, we conducted experiments across 20 benchmark datasets, encompassing node classification on homophily and heterophily graphs, link prediction, and graph classification. Notably, SimMLP proves highly effective in inductive and cold-start settings with unseen nodes, as well as in link prediction tasks where localized structures are crucial. In terms of inference efficiency, SimMLP demonstrates significant acceleration compared to GNNs (90$\sim$126$\times$) and other acceleration techniques (5$\sim$90$\times$). Our contributions are summarized as follows:
\begin{itemize}
  \item We introduce SimMLP, a self-supervised MLP learning method for graphs that aims to maximize the alignment between node features and graph structures in the embedding space.
  \item SimMLP is the first MLP learning method that is equivalent to GNNs in the optimal case. We provide a comprehensive theoretical analysis to demonstrate SimMLP's generalization capabilities and its equivalence to GNNs.
  \item We conduct extensive experiments to showcase the superiority of SimMLP, particularly in scenarios where structural insights are essential.
\end{itemize}

\section{Related Work}

\noindent\textbf{Graph Neural Networks} \citep{kipf2017semisupervised,hamilton2017inductive,velickovic2018graph,gilmer2017neural,liu2023fair,liu2024can,wang2024tackling,zhang2024diet} encode node embeddings following message passing framework. Basic GNNs include GCN \citep{kipf2017semisupervised}, GraphSAGE \citep{hamilton2017inductive}, GAT \citep{velickovic2018graph}, and so forth. For repid inference, SGC \citep{wu2019simplifying} and APPNP \citep{gasteiger2018combining} decompose feature transformation and neighborhood aggregation, which are recently proven to be expressive \citep{han2023mlpinit,yang2023graph}. Despite their success, the neighborhood dependency significantly constrains the inference speed.

\noindent\textbf{Self-Supervised Learning} (SSL) \citep{chen2020simple,he2020momentum} acts as a pre-training strategy for learning discriminative \citep{tian2020makes} and generalizable \citep{huang2023towards} representations without supervision. Numerous studies extend SSL on graphs to train GNNs \citep{velickovic2018deep,hassani2020contrastive,zhu2020deep,zhu2021graph,thakoor2022largescale,hou2022graphmae,sun2023all,wang2023heterogeneous,wang2024gft,wang2024select}. In particular, GCA \citep{zhu2021graph} extends instance discrimination \citep{chen2020simple} to align similar instances in two graph views, BGRL \citep{thakoor2022largescale} employs bootstrapping \citep{grill2020bootstrap} to further enhance training efficiency, and GraphACL \citep{xiao2024simple} leverages an asymmetric contrastive loss to encode structural information on graphs. However, the dependency on neighborhood information still limits the inference speed.

\begin{figure*}[!ht]
  \centering
  \includegraphics[width=0.85\linewidth]{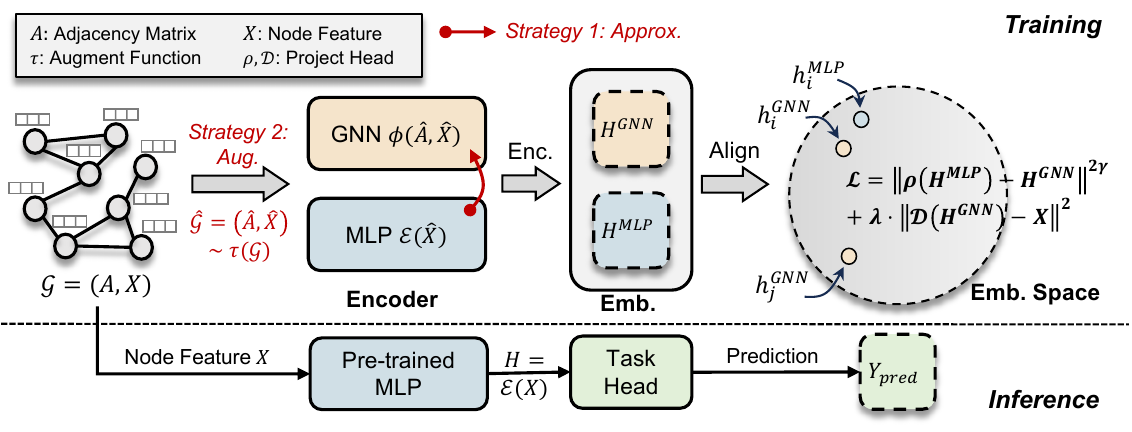}
  \caption{The overview of SimMLP. During pre-training, SimMLP uses GNN and MLP encoders to obtain node embeddings individually, and employs a self-supervised loss to maximize their alignment. To prevent the risk of trivial solutions, SimMLP further applies two strategies discussed in Section \ref{sec:prevent trivial solution}. During inference, SimMLP utilizes the pre-trained MLP to encode node embeddings, achieving significant acceleration by avoiding fetching neighborhood. }
  \label{fig:framework}
\end{figure*}

\noindent\textbf{Inference Acceleration on GNNs} encompasses quantization \citep{gupta2015deep,jacob2018quantization}, pruning \citep{han2015learning,frankle2018the}, and knowledge distillation (KD) \citep{hinton2015distilling,sun2019patient}. Quantization \citep{ding2021vq} approximates continuous data with limited discrete values, pruning \citep{zhou2021accelerating} involves dropping redundant neurons in the model, and KD focuses on transferring knowledge from large models to small models \citep{yan2020tinygnn}. However, they still need to fetch neighborhoods, resulting in constrained inference acceleration. Considering this, GLNN \citep{zhang2022graphless} utilizes structure-independent MLPs for predictions, significantly accelerating inference by eliminating message passing. However, the introduction of MLPs inevitably compromises the structural learning capability. Following works further integrate structural information into MLPs via positional embedding \citep{tian2022learning,wang2023graph}, label propagation \citep{yang2021extract}, neighborhood alignment \citep{hu2021graph,dong2022node,liu2022mlp,xiao2024efficient}, or motif cookbook \citep{yang2023vqgraph}, but these heuristic methods only consider one aspects of graph structures, failing to fully integrating structural knowledge. Unlike these methods, SimMLP is equivalent to GNNs in the optimal case, demonstrating better structural learning capabilities.

\section{Preliminary}

\noindent\textbf{Notations.} Considering a graph $\gG = (\mA, \mX)$ consisting of node set $\gV$ and edge set $E$, with $N$ nodes in total, we have node features $\mX \in \R^{N \times d}$ and a adjacent matrix $\mA \in \{0, 1\}^{N \times N}$, where $\mA_{ij} = 1$ iff $e_{i, j} \in E$, and $\mA_{ij}=0$ otherwise.
The GNN $\phi(\cdot, \cdot)$ takes node features $\vx_i$ and graph structure $\mA$ as input and outputs the structure-aware node embeddings $\vh_i^{GNN}$. The embedding follows a linear head to classify nodes into different classes, defined as:
\begin{equation}
  \vh_i^{GNN} = \phi(\vx_i, \mA), \quad \hat{\vy}_i^{GNN} = \text{head}^{GNN}(\vh_i^{GNN}).
\end{equation}
The GNNs highly rely on neighborhood information $\mA$, whereas neighborhood fetching poses considerable computational overhead during the inference. On the contrary, the MLP $\gE(\cdot)$ takes the node feature $\vx_i$ as input and output the node embeddings $\vh_i^{MLP}$, achieving fast inference by alleviating the neighborhood-fetching. The embeddings are then decoded via a prediction head:
\begin{equation}
  \vh_i^{MLP} = \gE(\vx_i), \quad \hat{\vy}_i^{MLP} = \text{head}^{MLP}(\vh_i^{MLP}).
\end{equation}
Although MLPs provide significantly faster inference over graph-structured datasets, the omitting of structural information inevitably sacrifices the model performance.

\noindent\textbf{Training MLPs on Graphs.} To jointly leverage the benefits of GNNs and MLPs, researchers propose methods to distill knowledge from pre-trained GNNs to MLPs by mimicking the predictions \citep{zhang2022graphless}. The training objective is defined as:
\begin{equation}
  \gL = \sum_{i \in \gV_{train}} \gL_{CE}(\hat{\vy}_i^{MLP}, \vy_i) + \lambda \sum_{i \in \gV} \gL_{KD}(\hat{\vy}_i^{MLP}, \hat{\vy}_i^{GNN}),
\end{equation}
where $\gL_{CE}$ is the cross-entropy between the prediction and ground-truth, and $\gL_{KD}$ optimizes the KL-divergence between predictions of teacher GNN and student MLP. During the inference, only the MLP is leveraged to encode node embeddings and make predictions, leading to a substantial inference acceleration. Despite this, the alignment in the label space maximizes the coarse-grained task-specific correlation between GNNs and MLPs, failing to capture the fine-grained and generalizable relationship between node features and graph structures \citep{Tian2020Contrastive}. To this end, SimMLP applies self-supervised learning to align GNNs and MLPs in a more intricate embedding space, better capturing structural information.

\section{Proposed SimMLP}

\subsection{Framework}

We present \textbf{SimMLP}: a \textbf{\textbf{S}}elf-superv\textbf{i}sed fra\textbf{m}ework for learning \textbf{MLP}s on graphs. The framework consists of three components: (1) \textit{GNN encoder}, (2) \textit{MLP encoder}, and (3) \textit{alignment loss}. As illustrated in Figure \ref{fig:framework}, SimMLP maximizes the alignment between GNN and MLP via a self-supervised loss. Specifically, given a graph $\gG = (\mA, \mX)$, we use GNN encoder $\phi(\cdot, \cdot)$ to extract structure-aware GNN embeddings $\vh_i^{GNN}$ and MLP encoder $\gE(\cdot)$ to obtain structure-free MLP embeddings $\vh_i^{MLP}$. The choice of GNN encoder is arbitrary; we can use different encoder for adopting to different tasks. For alignment, we employ the loss function to pretrain the model:
\begin{equation}
  \gL = \sum_{i \in \gV} \underbrace{   {\| \rho(\vh_i^{MLP}) - \vh_i^{GNN} \|}^{2\gamma}   }_{invariance}  + \lambda \underbrace{    \|\gD(\vh_i^{GNN}) - \vx_i \|^2     }_{reconstruction},
  \label{eq:objective}
\end{equation}
where $\gamma \ge 1$ serves as a scaling term, akin to an adaptive sample reweighing technique \citep{lin2017focal}, and $\lambda$ denotes the trade-off coefficient. The projector $\rho(\cdot), \gD(\cdot)$ can either be identity or learnable; we opt for a non-linear MLP to enhance the expressiveness in estimating instance distances \citep{chen2020simple}. The invariance term ensures the alignment between GNN and MLP embeddings \citep{grill2020bootstrap}, modeling the fine-grained and generalizable correlation between node features and localized graph structures. The reconstruction term acts as a regularizer to prevent the potential distribution shift \citep{batson2019noise2self}, providing better signals for training MLPs. In downstream tasks, we further train a task head for classification, as shown in Figure \ref{fig:framework}.

\begin{proposition}
  \label{thm:objective}
  Suppose $\gG = (\mA, \mX)$ is sampled from a latent graph $\gG_\gI = (\mA, \mF)$, $\gG \sim P(\gG_\gI)$, and $\mF^*$ is the lossless compression of $\mF$ that $\E[\mX | \mA, \mF^*] = \mF$. Let $\rho$ be an identity projector, and $\lambda = 1, \gamma=1$. The optimal MLP encoder $\gE^*$ satisfies
  \begin{align}
    \label{eq:thm objective}
    \gE^* =     & \argmin_{\gE} \: \E \left[ \left\| \mH^{MLP} - \mF^* \right\|^2 + \left\| \mH^{GNN} - \mF^* \right\|^2  \right]                         \\
    \nonumber + & \E \left[ \left\| \gD(\mH^{GNN}) - \mX \right\|^2  \right] - 2 \E_{\mF^*} \left[ \sum_i \Cov(\mH^{MLP}_i, \mH^{GNN}_i) | \mF^* \right].
  \end{align}
\end{proposition}

\begin{proof}
  All proofs in the paper are presented in Appendix \ref{sec:proof}.
\end{proof}

That is, the alignment loss implicitly enables the learned embeddings invariant to latent variables \citep{muthen2004latent,xie2022self}, and maximizes the consistency between GNNs and MLPs in covariance.

\subsection{Preventing Model Collapse}
\label{sec:prevent trivial solution}

\noindent\textbf{Challenges.} Training MLPs on graphs without supervision is a non-trivial task. As illustrated in Figure \ref{fig:model collapse}, naively applying the basic loss function (Equation \ref{eq:objective}) to align GNNs and MLPs results in model collapse, as evidenced by lower training loss and reduced accuracy. Consistent with our findings, \citet{xiao2024efficient} also show that simply employing the InfoNCE loss results in reduced performance.

\noindent\textbf{Causes.} We consider the issue derives from the heterogeneity in information sources, specifically node features and localized structures, encoded by MLPs and GNNs respectively. Before delving into the root causes, it is important to note that existing self-supervised methods primarily focus on aligning different aspects of a homogeneous information source (i.e., different views of the same graph). These methods typically use a single encoder to map various graph views into a unified embedding space, applying self-supervised loss to align the distances between views. This approach facilitates the learning of informative and generalizable embeddings. However, in scenarios involving heterogeneous information sources, each source often requires a distinct encoder, leading to projections into separate embedding spaces. Consequently, the self-supervised objective fails to accurately measure the true distance between sources, resulting in non-informative embeddings. We propose two strategies to address this issue.

\begin{figure}[!t]
  \centering
  \subfloat{\includegraphics[width=0.45\linewidth]{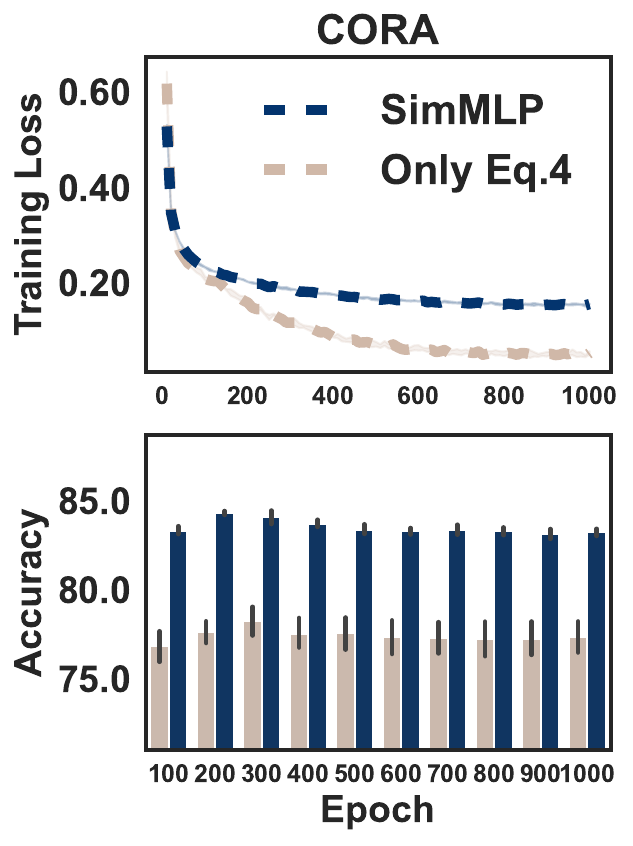}}
  \subfloat{\includegraphics[width=0.45\linewidth]{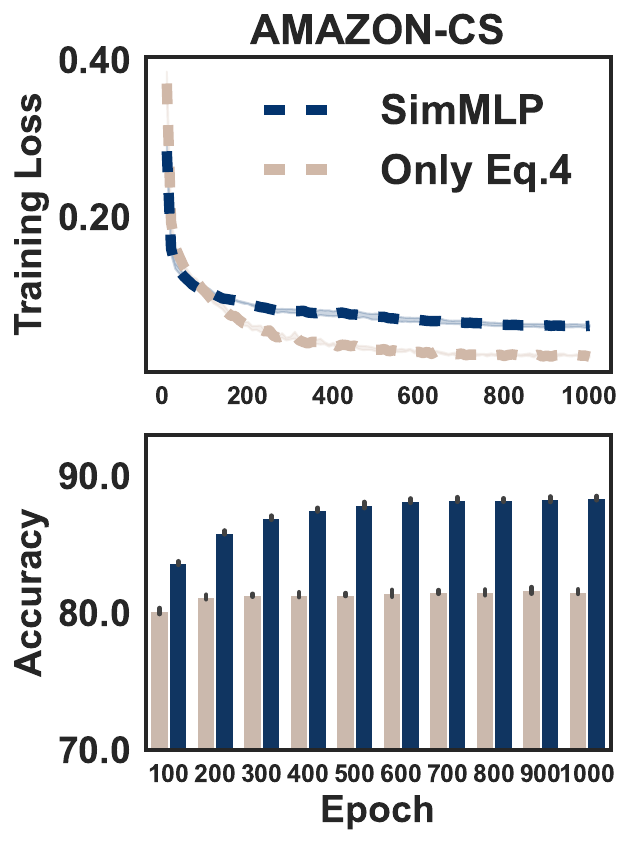}}
  \caption{Model collapse happens if naively applying the alignment loss (Equation \ref{eq:objective}). The strategies proposed in Sec. \ref{sec:prevent trivial solution} prevents the model collapse. }
  \label{fig:model collapse}
\end{figure}

\noindent\textbf{Strategy 1: Enhancing Consistency between GNN and MLP Encoders.} The challenge of handling heterogeneous information sources primarily stems from using different encoders. A straightforward solution is to use a single encoder to process all sources, including node features and localized structures. Fortunately, this approach is feasible for graphs. In the learning process of a GCN, neighborhood information is iteratively aggregated and updated, where aggregation can be viewed as a non-parametric averaging of neighborhoods, and updating as a parametric non-linear transformation. Thus, the learning process of GNNs can be approximated using MLPs by (1) applying an MLP encoder to node features to obtain MLP embeddings $\vh^M_i$, and (2) aggregating the MLP embeddings of neighboring nodes to approximate GNN embeddings $\vh^G_i$:
\begin{align}
  \text{Approx.: } & \vh_i^G = \sigma\left(\;\, \vh_i^M + \textstyle\sum_{(i,j) \in E} \alpha_{ij} \; \vh_j^M\right) ,
  \label{eq:approx}
\end{align}
where $\alpha_{ij}$ denotes the aggregation weight \citep{kipf2017semisupervised}, and $\sigma(\cdot)$ is the activation function. This approach allows the use of a single MLP encoder for both node features and localized structures, ensuring that $\vh^M_i$ and $\vh^G_i$ reside in the same embedding space. The form of this approximation is similar to SGC \citep{wu2019simplifying} and APPNP \citep{gasteiger2018combining}, which decompose feature transformation and message passing to reduce model complexity and redundant computations. However, SimMLP applies this strategy specifically to enhance the consistency between GNN and MLP encoders. Notably, this approximation can be adapted to various GNN architectures with simple modifications.

\noindent\textbf{Strategy 2: Enhancing Data Diversity through Augmentation.} Increasing data diversity \citep{lee2020self} is beneficial for handling heterogeneous information sources by creating multiple pairs of the same instances (i.e., target nodes and localized structures). We use augmentation techniques \citep{you2020graph,zhao2021data} to generate multiple views of nodes, allowing a single node to be associated with various pairs of node features and localized structures. The augmentation is defined as:
\begin{equation}
  \hat\gG = (\hat\mA, \hat\mX) \sim t(\gG), s.t., t(\gG) = \langle q_e(\mA), q_f(\mX) \rangle,
  \label{eq:aug}
\end{equation}
where $t(\cdot)$ represents the augmentation function, which includes structural augmentation $q_e(\cdot)$ and node feature augmentation $q_f(\cdot)$. For simplicity, we apply random edge masking and node feature masking with pre-defined augmentation ratios \citep{zhu2020deep}. The application of these two strategies prevents model collapse, as shown in Figure \ref{fig:model collapse}. It is important to note that these strategies are only employed during pre-training and do not affect the inference phase.

\subsection{Theoretical Understanding}
\label{sec:theory}

\begin{table}[!t]
  \centering
  \caption{The learning objective from the perspective of mutual information maximization. SimMLP is equivalent to GNNs in the optimal case, whileas other MLP methods lack the capability of (fully) leveraging localized structures.}
  \resizebox{\linewidth}{!}{
    \begin{tabular}{cll}
      \toprule
                                           & \textbf{Methods}                                                 & \textbf{Learning Objective}                                                                                                                                                                                               \\ \midrule
                                           & GNNs \citep{kipf2017semisupervised,hamilton2017inductive}        & \begin{minipage}{5cm} $\sum_{i \in \gV} I(\vy_i; \gS_i) = \sum_{i \in \gV} I(\vy_i; \mX^{[i]}) + \sum_{i \in \gV} I(\vy_i; \mA^{[i]} | \mX^{[i]})$\end{minipage} \\ \midrule
      \multirow{6}{*}{\rotatebox{90}{MLP}} & MLP \citep{zhang2022graphless}                                   & $\sum_{i \in \gV} I(\vx_i; \vy_i)$                                                                                                                                                                                        \\
                                           & GraphMLP \citep{hu2021graph}, GraphECL \citep{xiao2024efficient} & $\sum_{i \in \gV} I(\vy_i; \vx_i) + I(\vx_i; \mX^{[i]})$                                                                                                                                                                  \\
                                           & GLNN \citep{zhang2022graphless}                                  & $\sum_{i \in \gV} I(\vx_i; \vy_i | \gS_i) + I(\vx_i; \vy_i)$                                                                                                                                                              \\ \cmidrule{2-3}
                                           & NOSMOG \citep{tian2022learning}, GENN \citep{wang2023graph}      & $\sum_{i \in \gV} I(\vx_i; \vy_i | \gS_i) + I(\vx_i; \vy_i) + I(\mA; \vy_i)$                                                                                                                                              \\
                                           & VQGraph \citep{yang2023vqgraph}                                  & $\sum_{i \in \gV} I(\vx_i; \vy_i | \gS_i) + I(\vx_i; \vy_i) + I(\mA; \vx_i)$                                                                                                                                              \\ \cmidrule{2-3}
      \rowcolor{Gray}                      & \textbf{SimMLP}                                                  & $\sum_{i \in \gV} I(\vy_i; \vx_i) + I(\vx_i; \gS_i)$                                                                                                                                                                      \\ \bottomrule
    \end{tabular}
  }
  \label{tab:MI}
\end{table}

\noindent\textbf{Mutual Information Maximization.} Mutual information $I(\cdot; \cdot)$ is a concept in information theory, measuring the mutual dependency between two random variables. This has been widely used in signal processing and machine learning \citep{belghazi2018mutual}. Intuitively, maximizing the mutual information between two variables can increase their correlation (i.e., decrease their uncertainty). In this section, we interpret SimMLP and existing MLP learning methods from the perspective of mutual information maximization to analyze their learning objectives, as summarized in Table \ref{tab:MI}. Firstly, we unify the notations and introduce the key lemmas. A graph $\gG = (\mX, \mA, \mY)$ consists of node features $\mX$, graph structure $\mA$ and node labels $\mY$. We define ego-graph around node $i$ as $\gS_i = (\mX^{[i]}, \mA^{[i]})$, where $\mX^{[i]}$, $\mA^{[i]}$ denote node features and graph structure of $\gS_i$, respectively.

\begin{lemma}
  \label{thm:CE MI}
  Minimizing the cross-entropy $H(\mY; \hat{\mY} | \mX)$ is equivalent to maximizing the mutual information $I(\mX; \mY)$.
\end{lemma}

MLP minimizes the cross-entropy between ground-truth $\mY$ and the predictions relying on node features, i.e., $\hat{\mY} | \mX$, which is equivalent to maximize the mutual information $\sum_{i \in \gV} I(\vy_i; \vx_i)$. GNNs apply message passing to use localized structures in making predictions, which aims to maximize $\sum_{i \in \gV} I(\vy_i; \gS_i) = \sum_{i \in \gV} I(\vy_i; \mX^{[i]}) + \sum_{i \in \gV} I(\vy_i; \mA^{[i]} | \mX^{[i]})$. GraphMLP and GraphECL maximize the consistency between the target node and its neighborhood encoded by MLPs $I(\vx_i; \mX^{[i]})$, failing to learning the intrinsic correlation between node features and graph structures. GLNN distills knowledge from GNNs to MLPs to maximize $\sum_{i \in \gV} I(\vx_i; \vy_i | \gS_i) + I(\vx_i; \vy_i)$, where $\vy_i | \gS_i$ denotes the soft label from GNNs, but fails to explicitly utilize structural information in making predictions. Following GLNN, GENN and NOSMOG employ positional embeddings to leverage structural information by optimizing $I(\mA; \vy_i)$, and VQGraph uses codebook to inject structural knowledge on node features to further optimize $I(\mA; \vx_i)$. However, these methods cannot fully leverage the localized structures in predictions, leading to sub-optimal performance.

SimMLP employs self-supervised learning to maximize the mutual information between GNNs and MLPs \citep{bachman2019learning}. The objective is to maximize $\sum_{i \in \gV} I(\vy_i; \vx_i) + I(\vx_i; \gS_i)$. The first term optimizes the model on downstream tasks, corresponding to task-specific prediction head. The second term is the training objective of SimMLP (Equation \ref{eq:objective}) that denotes the alignment between GNNs and MLPs. When the second term is maximized, $\vx_i$ would preserve all information of $\gS_i$, turning the overall objective to maximize $\sum_{i \in \gV} I(\vy_i; \gS_i)$. This demonstrates the equivalence between SimMLP (in the optimal case) and GNNs, showing the superiority of SimMLP in leveraging graph structures for predictions. Our analysis aligns with \citet{chen2021on} and \citet{zhang2022graphless} that the expressiveness of MLPs on node classification task is bounded by the induced ego-graphs $\gS_i$. Despite our analysis is based on node classification, it is readily to be extended to link prediction or graph classification.

\noindent\textbf{Information Bottleneck Principle.} Information bottleneck \citep{tishby2000information,tishby2015deep} focuses on finding the optimal compression of observed random variables by achieving the trade-off between informativeness and generalization \citep{shwartz2017opening}. For example, given a random variable $\mX$ sampled from latent variable $\mY$, the aim is to find the optimal compression $\mZ^* = \argmin_{\mZ} I(\mX; \mZ) - \beta I(\mZ; \mY)$. Intuitively, minimizing $I(\mX; \mZ)$ aims to obtain the minimum compression, and maximizing $I(\mZ; \mY)$ preserves the essential information of $\mY$. For SimMLP, we assume the observed graph $\gG$ is sampled from latent graph $\gG_\gI$ (Proposition \ref{thm:objective}), and aim to compress the graph into $\mT = (\mH^{MLP}, \mH^{GNN})$.
\begin{proposition}
  \label{thm:information bottleneck}
  The optimal compression $\mT^*$ satisfies
  \begin{align}
    \mT^*
    = & \argmin_{\mH^{M}, \mH^{G}} \lambda H(\mH^{M} | \gG_\gI)  + H(\mH^{G})  + \lambda H(\mH^{G} | \gG_\gI) + H(\mH^{M} | \mH^{G}),
    \label{eq:information bottleneck}
  \end{align}
  where $\lambda = \frac{\beta}{1 - \beta} > 0$, $\mH^M$ and $\mH^G$ indicate $\mH^{MLP}$ and $\mH^{GNN}$.
\end{proposition}
$I(\cdot ; \cdot)$ and $H(\cdot)$ denote mutual information and entropy, respectively. Intuitively, the optimal compression (Equation \ref{eq:information bottleneck}) is attainable with the optimal encoder $\gE^*$ (Equation \ref{eq:thm objective}). In particular, minimizing $H(\mH^{MLP} | \gG_\gI)$ and $H(\mH^{GNN} | \gG_\gI)$ preserves the latent information in GNN and MLP embeddings, which can be instantiated as minimizing $\| \mH^{MLP} - \mF^* \|^2 + \| \mH^{GNN} - \mF^* \|^2$. In addition, the minimum conditional entropy $H(\mH^{MLP} | \mH^{GNN})$ denotes the alignment between GNN and MLP embeddings, which could be modeled as maximizing $\sum_i \Cov(\mH^{MLP}_i, \mH^{GNN}_i)$. Furthermore, minimizing entropy $H(\mH^{GNN})$ reduces the uncertainty of the GNN embeddings, which could be achieved by preserving more node feature information, i.e., minimizing $\| \gD(\mH^{GNN}) - \mX \|^2$. The analysis bridges the information bottleneck and the objective function of SimMLP, showing the potential in learning generalizable embeddings \citep{alemi2017deep}.

\begin{table}[!t]
  \centering
  \caption{SimMLP shares two inductive biases with GNNs, i.e., homophily and local structure importance, which are measured by \textbf{smoothness} and \textbf{min-cut}, respectively.}
  \label{tab:inductive bias}
  \resizebox{\linewidth}{!}{
    \begin{tabular}{l  ccc | ccc}
      \toprule
      ~                                  & \multicolumn{3}{c}{Smoothness$\downarrow$} & \multicolumn{3}{c}{Min-Cut$\uparrow$}                                                                     \\ \cmidrule(lr){2-4}\cmidrule(lr){5-7}
      Methods                            & Cora                                       & Citeseer                              & Pubmed         & Cora           & Citeseer       & Pubmed         \\ \midrule\midrule
      Raw Node Feature                   & 0.822                                      & 0.783                                 & 0.734          & $-$            & $-$            & $-$            \\ \midrule\midrule
      SAGE \citep{hamilton2017inductive} & 0.113                                      & 0.184                                 & 0.143          & 0.924          & 0.943          & 0.918          \\
      BGRL \citep{thakoor2022largescale} & 0.155                                      & 0.102                                 & 0.333          & 0.885          & 0.935          & 0.856          \\ \midrule\midrule
      MLP \citep{zhang2022graphless}     & 0.463                                      & 0.444                                 & 0.485          & 0.666          & 0.804          & 0.863          \\
      GLNN \citep{zhang2022graphless}    & 0.282                                      & 0.268                                 & 0.421          & 0.886          & 0.916          & 0.793          \\
      NOSMOG \citep{tian2022learning}    & 0.267                                      & 0.230                                 & 0.394          & 0.902          & 0.932          & 0.834          \\
      VQGraph \citep{yang2023vqgraph}    & 0.253                                      & 0.212                                 & 0.396          & 0.914          & 0.940          & 0.831          \\  \midrule\midrule
      \rowcolor{Gray}\textbf{SimMLP}     & \textbf{0.196}                             & \textbf{0.170}                        & \textbf{0.360} & \textbf{0.934} & \textbf{0.958} & \textbf{0.886} \\ \bottomrule
    \end{tabular}
  }
\end{table}

\begin{table*}[!t]
  \centering
  \caption{Node classification accuracy in transductive setting, where the overall best and the sub-category best are indicated by \textbf{bold} and \underline{underline}, respectively. We report the mean and standard deviation of ten runs with different random seeds. }
  \label{tab:transductive}
  \resizebox{\linewidth}{!}{
    \begin{tabular}{cl | c | c|c|c|c|c|c|c|c|c | c}
      \toprule
                                           & Methods                              & Cora                          & Citeseer                      & Pubmed                        & Computer                      & Photo                         & Co-CS                         & Co-Phys                       & Wiki-CS                       & Flickr                        & Arxiv                         & \textbf{\textit{Avg.}} \\ \midrule\midrule
      \multirow{4}{*}{\rotatebox{90}{GNN}} & SAGE \citep{hamilton2017inductive}   & 81.4{\small ±0.9}             & \underline{70.4{\small ±1.4}} & \underline{85.9{\small ±0.4}} & 88.9{\small ±0.3}             & \underline{93.8{\small ±0.4}} & 93.4{\small ±0.2}             & \underline{95.7{\small ±0.1}} & 80.9{\small ±0.6}             & 48.5{\small ±0.8}             & \textbf{72.1{\small ±0.3}}    & \underline{81.1}       \\
                                           & GAT \citep{velickovic2018graph}      & \underline{82.3{\small ±1.2}} & 68.9{\small ±1.5}             & 84.7{\small ±0.4}             & \textbf{89.9{\small ±0.5}}    & 91.9{\small ±0.4}             & 92.0{\small ±0.3}             & 95.1{\small ±0.2}             & 80.0{\small ±0.6}             & 51.4{\small ±0.2}             & 71.8{\small ±0.4}             & 80.8                   \\
                                           & APPNP \citep{gasteiger2018combining} & 75.5{\small ±1.6}             & 68.1{\small ±1.2}             & 84.6{\small ±0.3}             & 87.4{\small ±0.3}             & 93.4{\small ±0.5}             & \underline{94.6{\small ±0.2}} & 95.4{\small ±0.1}             & 79.1{\small ±0.3}             & 47.5{\small ±0.3}             & 71.0{\small ±0.2}             & 79.7                   \\
                                           & SGC \citep{wu2019simplifying}        & 81.8{\small ±0.9}             & 69.0{\small ±1.6}             & 85.3{\small ±0.3}             & 89.3{\small ±0.6}             & 92.7{\small ±0.4}             & 94.0{\small ±0.2}             & 94.8{\small ±0.2}             & \underline{81.1{\small ±0.6}} & \textbf{51.8{\small ±0.2}}    & 70.0{\small ±0.4}             & 81.0                   \\ \midrule\midrule
      \multirow{5}{*}{\rotatebox{90}{GCL}} & DGI \citep{velickovic2018deep}       & 82.3{\small ±0.6}             & 71.8{\small ±0.7}             & 76.8{\small ±0.7}             & 80.0{\small ±0.2}             & 91.6{\small ±0.2}             & 92.2{\small ±0.5}             & 94.5{\small ±0.0}             & 76.4{\small ±0.6}             & 46.9{\small ±0.1}             & 70.1{\small ±0.2}             & 78.3                   \\
                                           & MVGRL \citep{hassani2020contrastive} & \underline{83.9{\small ±0.5}} & \underline{72.1{\small ±1.3}} & \underline{86.3{\small ±0.6}} & 87.9{\small ±0.3}             & 91.9{\small ±0.2}             & 92.2{\small ±0.1}             & 95.3{\small ±0.0}             & 77.6{\small ±0.1}             & 49.3{\small ±0.1}             & 70.9{\small ±0.1}             & 80.7                   \\
                                           & GRACE \citep{zhu2020deep}            & 80.5{\small ±1.0}             & 65.5{\small ±2.1}             & 84.6{\small ±0.5}             & \underline{88.4{\small ±0.3}} & \underline{92.8{\small ±0.6}} & 93.0{\small ±0.3}             & 95.4{\small ±0.1}             & \underline{78.6{\small ±0.5}} & 49.3{\small ±0.1}             & \underline{71.0{\small ±0.1}} & 79.5                   \\
                                           & GCA \citep{zhu2021graph}             & 83.5{\small ±0.5}             & 71.3{\small ±0.2}             & 86.0{\small ±0.4}             & 87.4{\small ±0.3}             & 92.6{\small ±0.2}             & \underline{93.1{\small ±0.0}} & \underline{95.7{\small ±0.0}} & 78.4{\small ±0.1}             & 49.0{\small ±0.1}             & 70.9{\small ±0.1}             & \underline{80.8}       \\
                                           & BGRL \citep{thakoor2022largescale}   & 81.3{\small ±0.6}             & 66.9{\small ±0.6}             & 84.9{\small ±0.2}             & 88.2{\small ±0.2}             & 92.5{\small ±0.1}             & 92.1{\small ±0.1}             & 95.2{\small ±0.1}             & 77.5{\small ±0.8}             & \underline{49.7{\small ±0.1}} & 70.8{\small ±0.1}             & 79.9                   \\ \midrule\midrule
      \multirow{8}{*}{\rotatebox{90}{MLP}} & MLP \citep{zhang2022graphless}       & 64.5{\small ±1.9}             & 64.0{\small ±1.3}             & 80.7{\small ±0.3}             & 80.8{\small ±0.3}             & 87.8{\small ±0.5}             & 91.7{\small ±0.3}             & 95.1{\small ±0.1}             & 75.2{\small ±0.5}             & 46.2{\small ±0.1}             & 56.4{\small ±0.3}             & 74.2                   \\
                                           & GraphMLP \citep{hu2021graph}         & 79.5{\small ±0.8}             & 72.1{\small ±0.5}             & 84.3{\small ±0.2}             & 84.0{\small ±0.6}             & 90.9{\small ±1.0}             & 90.4{\small ±0.6}             & 93.5{\small ±0.2}             & 76.4{\small ±0.5}             & 46.3{\small ±0.2}             & 63.4{\small ±0.2}             & 78.1                   \\
                                           & GLNN \citep{zhang2022graphless}      & 81.3{\small ±1.2}             & 71.2{\small ±0.7}             & \underline{86.3{\small ±0.5}} & 87.5{\small ±0.6}             & 93.9{\small ±0.3}             & \underline{94.2{\small ±0.2}} & 95.4{\small ±0.1}             & \underline{80.7{\small ±0.7}} & 46.2{\small ±0.2}             & 64.0{\small ±0.5}             & 80.1                   \\
                                           & GENN \citep{wang2023graph}           & 82.1{\small ±0.8}             & 71.4{\small ±1.3}             & 86.3{\small ±0.3}             & 87.1{\small ±0.6}             & 93.6{\small ±0.7}             & 93.8{\small ±0.3}             & 95.5{\small ±0.1}             & 80.5{\small ±0.7}             & 46.4{\small ±0.3}             & 70.1{\small ±0.6}             & 80.7                   \\
                                           & VQGraph \citep{yang2023vqgraph}      & \underline{82.3{\small ±0.6}} & \underline{73.0{\small ±1.2}} & 86.0{\small ±0.4}             & 87.5{\small ±0.7}             & 93.9{\small ±0.3}             & 93.8{\small ±0.1}             & 95.6{\small ±0.1}             & \underline{79.9{\small ±0.2}} & \underline{47.0{\small ±0.2}} & 70.8{\small ±0.8}             & 81.0                   \\
                                           & NOSMOG \citep{tian2022learning}      & 82.3{\small ±1.1}             & 72.4{\small ±1.3}             & 86.2{\small ±0.3}             & \underline{87.6{\small ±1.1}} & \underline{93.9{\small ±0.5}} & 93.8{\small ±0.2}             & \underline{95.7{\small ±0.1}} & 80.5{\small ±0.8}             & 46.7{\small ±0.3}             & \underline{70.8{\small ±0.4}} & \underline{81.0}       \\
      \rowcolor{Gray}                      & \textbf{SimMLP}                      & \textbf{84.6{\small ±0.2}}    & \textbf{73.5{\small±0.5}}     & \textbf{87.0{\small ±0.1}}    & \underline{88.5{\small ±0.2}} & \textbf{94.3{\small ±0.1}}    & \textbf{94.9{\small ±0.1}}    & \textbf{96.2{\small ±0.0}}    & \textbf{81.2{\small ±0.1}}    & \underline{49.9{\small ±0.1}} & \underline{71.1{\small ±0.1}} & \textbf{82.1}          \\ \bottomrule
    \end{tabular}
  }
\end{table*}

\begin{table}[!t]
  \centering
  \caption{Node classification accuracy in inductive settings. }
  \label{tab:production}
  \resizebox{\linewidth}{!}{
    \begin{tabular}{l cccc}
      \toprule
      Methods                            & Cora                          & Citeseer                      & Pubmed                        & Arxiv                         \\ \midrule\midrule
      SAGE \citep{hamilton2017inductive} & 77.5{\small ±1.8}             & \underline{68.4{\small ±1.6}} & \underline{85.0{\small ±0.4}} & 68.5{\small ±0.6}             \\
      BGRL \citep{thakoor2022largescale} & \underline{77.7{\small ±1.1}} & 64.3{\small ±1.6}             & 84.0{\small ±0.5}             & \underline{69.3{\small ±0.4}} \\ \midrule\midrule
      MLP \citep{zhang2022graphless}     & 63.8{\small ±1.7}             & 64.0{\small ±1.2}             & 80.9{\small ±0.5}             & 55.9{\small ±0.5}             \\
      GLNN \citep{zhang2022graphless}    & 78.3{\small ±1.0}             & 69.6{\small ±1.1}             & \underline{85.5{\small ±0.5}} & 63.5{\small ±0.5}             \\
      GENN \citep{wang2023graph}         & 77.8{\small ±1.6}             & 67.3{\small ±1.5}             & 84.3{\small ±0.5}             & 68.5{\small ±0.5}             \\
      VQGraph \citep{yang2023vqgraph}    & \underline{78.4{\small ±1.8}} & \underline{70.4{\small ±1.1}} & 85.4{\small ±0.6}             & \underline{69.3{\small ±0.9}} \\
      NOSMOG \citep{tian2022learning}    & 77.8{\small ±1.9}             & 68.6{\small ±1.4}             & 83.8{\small ±0.5}             & 69.1{\small ±0.8}             \\
      GraphECL \citep{xiao2024efficient} & 77.8{\small ±1.3}             & 69.2{\small ±1.2}             & 84.5{\small ±0.5}             & 69.1{\small ±0.6}             \\ \midrule\midrule
      \rowcolor{Gray} \textbf{SimMLP}    & \textbf{81.4{\small ±1.2}}    & \textbf{72.3{\small ±0.9}}    & \textbf{86.5{\small ±0.3}}    & \textbf{70.2{\small ±0.5}}    \\
      \bottomrule
    \end{tabular}
  }
\end{table}

\noindent\textbf{Inductive Bias.} To further analyze the equivalence between SimMLP and GNNs, we investigate whether SimMLP and GNNs have similar inductive biases. We consider SimMLP has two key inductive biases, i.e., homophily philosophy and local structure importance, as shown in Table \ref{tab:inductive bias} (where more resutls are in Appendix \ref{sec:ind bias full}). Homophily implies topologically close nodes have similar properties, which could be naturally incorporated in message passing \citep{li2022finding} that updates node embeddings based on neighborhoods. However, MLP-based methods cannot (or partially) leverage homophily bias due to the lack of structural learning ability. To evaluate the homophily, we measure the distance between embeddings of directly connected nodes, which corresponds to \textit{smoothness}. In particular, we instantiate this as Mean Average Distance (MAD) \citep{chen2020measuring}:
\begin{equation}
  \text{MAD} = \frac{\sum_{i \in \gV} \sum_{j \in \gN(i)}(\vh_i^{MLP} - \vh_j^{MLP})^2}{\sum_{i \in \gV} \sum_{j \in \gN(i)} \mathbf{1}}.
\end{equation}
Intuitively, a low smoothness value indicates a high similarity between directly connected nodes, demonstrating the capability to leverage graph structural information \citep{hou2019measuring}. As shown in Table \ref{tab:inductive bias}, compared to GNNs, MLP-based methods fall short in reducing the distance between topologically close nodes, even for NOSMOG and VQGraph that directly integrate graph structures. SimMLP goes beyond these methods by aligning GNNs and MLPs in a intricate embedding space, approaching to GNNs.

\begin{table}[!t]
  \centering
  \caption{Node classification accuracy under cold-start setting. }
  \label{tab:cold-start}

  \resizebox{\linewidth}{!}{
    \begin{tabular}{lcccc}
      \toprule
      Methods                            & Cora                          & Citeseer                      & Pubmed                        & Arxiv                         \\ \midrule\midrule
      SAGE \citep{hamilton2017inductive} & 69.7{\small ±2.9}             & \underline{67.1{\small ±2.6}} & 82.9{\small ±1.0}             & 55.5{\small ±0.8}             \\
      BGRL \citep{thakoor2022largescale} & \underline{79.4{\small ±1.7}} & 65.0{\small ±2.2}             & \underline{84.0{\small ±1.0}} & \underline{65.0{\small ±0.5}} \\ \midrule\midrule
      MLP \citep{zhang2022graphless}     & 64.2{\small ±2.1}             & 64.4{\small ±1.8}             & 80.9{\small ±0.7}             & 55.9{\small ±0.7}             \\
      GLNN \citep{zhang2022graphless}    & \underline{72.0{\small ±1.7}} & 69.1{\small ±2.6}             & 84.4{\small ±0.9}             & 60.6{\small ±0.6}             \\
      GENN \citep{wang2023graph}         & 68.1{\small ±2.2}             & 65.1{\small ±2.8}             & 78.4{\small ±0.8}             & 62.6{\small ±0.7}             \\
      VQGraph \citep{yang2023vqgraph}    & 70.4{\small ±3.4}             & 70.0{\small ±1.6}             & \underline{84.5{\small ±1.5}} & \underline{64.0{\small ±1.7}} \\
      NOSMOG \citep{tian2022learning}    & 68.1{\small ±3.0}             & 67.1{\small ±2.1}             & 77.4{\small ±0.8}             & 63.5{\small ±0.8}             \\
      GraphECL \citep{xiao2024efficient} & 71.5{\small ±4.2}             & \underline{70.9{\small ±2.4}} & 82.4{\small ±0.9}             & 63.3{\small ±0.7}             \\ \midrule\midrule
      \rowcolor{Gray} \textbf{SimMLP}    & \textbf{80.5{\small ±2.2}}    & \textbf{72.8{\small ±1.6}}    & \textbf{86.4{\small ±0.5}}    & \textbf{66.1{\small ±1.1}}    \\
      \bottomrule
    \end{tabular}
  }
\end{table}

\begin{table*}[!t]
  \centering
  \caption{The performance on graph classification tasks with accuracy (\%).}
  \label{tab:graph classification}
  \resizebox{0.85\linewidth}{!}{
    \begin{tabular}{cl | c|c|c|c|c|c|c}
      \toprule
                                    & Methods                                  & IMDB-B                        & IMDB-M                        & COLLAB                        & PTC-MR                        & MUTAG                         & DD                            & PROTEINS                      \\ \midrule\midrule
      Supervised                    & GIN \citep{xu2018how}                    & 75.1{\small ±5.1}             & 52.3{\small ±2.8}             & 80.2{\small ±1.9}             & 64.6{\small ±1.7}             & 89.4{\small ±5.6}             & 74.9{\small ±3.1}             & 76.2{\small ±2.8}             \\ \midrule\midrule
      \multirow{2}{*}{Graph Kernel} & WL \citep{shervashidze2011weisfeiler}    & 72.3{\small ±3.4}             & 47.0{\small ±0.5}             & -                             & 58.0{\small ±0.5}             & 80.7{\small ±3.0}             & -                             & 72.9{\small ±0.6}             \\
                                    & DGK \citep{yanardag2015deep}             & 67.0{\small ±0.6}             & 44.6{\small ±0.5}             & -                             & 60.1{\small ±2.6}             & 87.4{\small ±2.7}             & -                             & 73.3{\small ±0.8}             \\ \midrule\midrule
      \multirow{5}{*}{GCL}          & graph2vec \citep{narayanan2017graph2vec} & 71.1{\small ±0.5}             & 50.4{\small ±0.9}             & -                             & 60.2{\small ±6.9}             & 83.2{\small ±9.3}             & -                             & 73.3{\small ±2.1}             \\
                                    & MVGRL \citep{hassani2020contrastive}     & 71.8{\small ±0.8}             & \underline{50.8{\small ±0.9}} & 73.1{\small ±0.6}             & -                             & \textbf{89.2{\small ±1.3}}    & 75.2{\small ±0.6}             & 74.0{\small ±0.3}             \\
                                    & InfoGraph \citep{Sun2020InfoGraph}       & \underline{73.0{\small ±0.9}} & 49.7{\small ±0.5}             & 70.7{\small ±1.1}             & \textbf{61.7{\small ±1.4}}    & \underline{89.0{\small ±1.1}} & 72.9{\small ±1.8}             & 74.4{\small ±0.3}             \\
                                    & GraphCL \citep{you2020graph}             & 71.1{\small ±0.4}             & 48.6{\small ±0.7}             & 71.4{\small ±1.2}             & -                             & 86.8{\small ±1.3}             & \textbf{78.6{\small ±0.4}}    & 74.4{\small ±0.5}             \\
                                    & JOAO \citep{you2021graph}                & 70.2{\small ±3.1}             & 49.2{\small ±0.8}             & 69.5{\small ±0.4}             & -                             & 87.4{\small ±1.0}             & -                             & \underline{74.6{\small ±0.4}} \\ \midrule\midrule
      \multirow{3}{*}{MLP}          & MLP$^*$                                  & 49.5{\small ±1.7}             & 33.1{\small ±1.6}             & 51.9{\small ±1.0}             & 54.4{\small ±1.4}             & 67.2{\small ±1.0}             & 58.6{\small ±1.4}             & 59.2{\small ±1.0}             \\
                                    & MLP + KD$^*$                             & 72.9{\small ±1.0}             & 48.1{\small ±0.5}             & \underline{75.4{\small ±1.5}} & 59.4{\small ±1.4}             & 87.4{\small ±0.7}             & 73.6{\small ±1.7}             & 73.5{\small ±1.8}             \\
      \rowcolor{Gray}               & \textbf{SimMLP}                          & \textbf{74.1{\small ±0.2}}    & \textbf{51.4{\small ±0.5}}    & \textbf{81.0{\small ±0.1}}    & \underline{60.3{\small ±1.1}} & 87.7{\small ±0.2}             & \underline{78.4{\small ±0.5}} & \textbf{75.3{\small ±0.1}}    \\ \bottomrule
    \end{tabular}
  }

  {\small The reported results of baselines are from previous papers if available \citep{you2020graph,you2021graph,hou2022graphmae}. $^*$ indicates the results are from our implementation.}
\end{table*}

Local structure importance describes the local neighborhoods preserve the crucial information for predictions. GNNs utilize localized information in making predictions, naturally emphasizing the localized structures, but MLPs generally take node features as input, failing to fully leverage structural information. Alternatively, it measures the alignment between localized structures and model predictions, aligning to the philosophy of \textit{Min-Cut} \citep{stoer1997simple}. In particular, denote predictions $\hat{\mY}$ as graph partitions, a high Min-Cut value indicates a high correlation between predictions and localized structures, as it implies high intra-partition connectivity and low inter-partition connectivity. We follow \citet{shi2000normalized} to model the Min-Cut problem:
\begin{equation}
  \text{Min-Cut} = tr(\hat\mY^T\mA\hat\mY) / tr(\hat\mY^T\mD\hat\mY),
\end{equation}
where $\mA$ and $\mD$ are adjacency matrix and diagonal node degree matrix, respectively. As shown in Table \ref{tab:inductive bias}, SimMLP demonstrates an inductive bias towards local structure importance, evidenced by the optimal average Min-Cut result.

\section{Experiments}

\begin{figure}[!t]
  \centering
  \includegraphics[width=\linewidth]{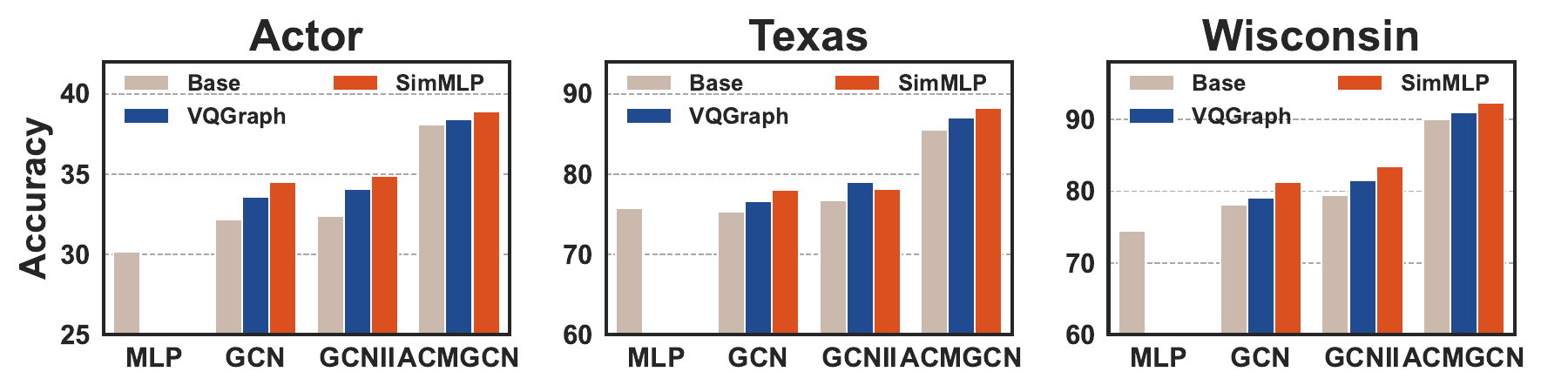}
  \caption{Node classification on heterophily graphs.}
  \label{fig:heterophily graph}
\end{figure}

\begin{figure}[!t]
  \centering
  \includegraphics[width=\linewidth]{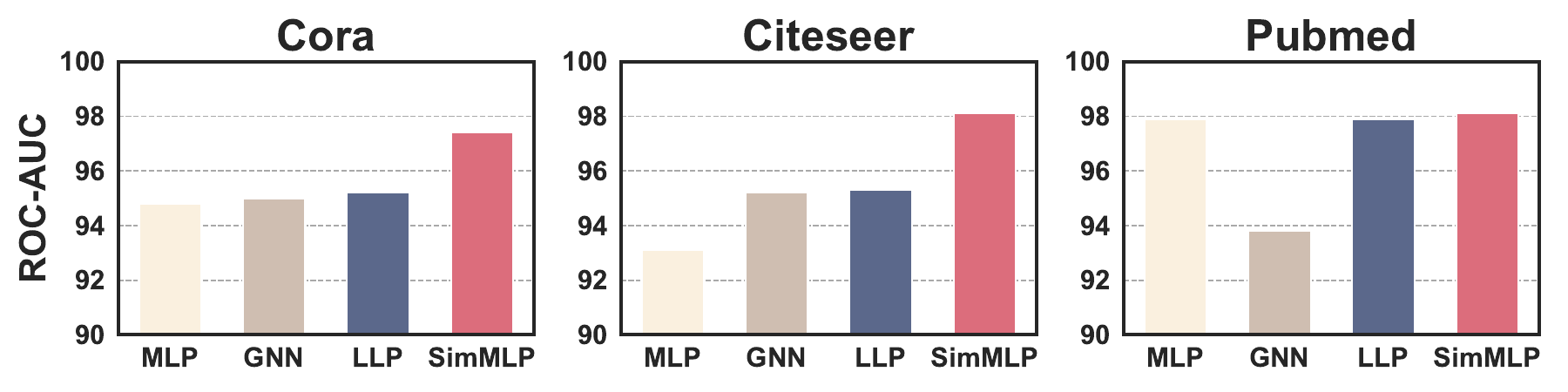}
  \caption{Link prediction performance.}
  \label{fig:link pred}
\end{figure}

\subsection{Experimental Setup}

We conduct experiments on node classification, link prediction, and graph classification. We use 20 datasets in total, including homophily graphs: Cora, Citeseer, Pubmed, Computer, Photo, Co-CS, Co-Phys, Wiki-CS, Flickr, and Arxiv; heterophily graphs: Actor, Texas, Wisconsin, and molecule/social networks (for graph classification): IMDB-B, IMDB-D, COLLAB, PTC-MR, MUTAG, DD, PROTEINS. More evaluation protocols are in Appendix \ref{sec:appendix exp setup}.

\subsection{Node Classification}

\noindent\textbf{Transductive Setting.} Given a graph $\gG = (\gV, E)$, all nodes $v \in \gV$ are visible during training, with the nodes partitioned into non-overlapping sets: $\gV = \gV_{train} \sqcup \gV_{val} \sqcup \gV_{test}$. We use $\gV_{train}$ for training and $\gV_{val}$ and $\gV_{test}$ for evaluation. The results, presented in Table \ref{tab:transductive}, show that SimMLP outperforms self-supervised GCL methods in all settings and surpasses supervised GNNs in 7 out of 10 datasets. Compared to other MLP-based methods, SimMLP achieves a superior average performance of 82.1, exceeding the second-best score of 81.0, thereby demonstrating its effectiveness in node classification. We present the comprehensive ablation study in Appendix \ref{sec:ablation}.

\begin{table}[!t]
  \captionof{table}{SimMLP achieves significant acceleration and improved performance over inference acceleration techniques. }
  \label{tab:inference acceleration}
  \resizebox{\linewidth}{!}{
    \begin{tabular}{l cc cc}
      \toprule
                                                 & \multicolumn{2}{c}{Flickr}  & \multicolumn{2}{c}{Arxiv}                                                \\ \cmidrule(lr){2-3}\cmidrule(lr){4-5}
      Methods                                    & Time (ms)                   & Acc.                      & Time (ms)                    & Acc.          \\ \midrule\midrule
      SAGE \citep{hamilton2017inductive}         & 80.7                        & 47.2                      & 314.7                        & 68.5          \\ \midrule\midrule
      SGC \citep{wu2019simplifying}              & 76.9 (1.1$\times$)          & 47.4                      & 265.9 (1.2$\times$)          & 68.9          \\
      APPNP \citep{gasteiger2018combining}       & 78.1 (1.0$\times$)          & 47.5                      & 284.1 (1.1$\times$)          & 69.1          \\ \midrule\midrule
      QSAGE \citep{zhang2022graphless}           & 70.6 (1.1$\times$)          & 47.2                      & 289.5 (1.1$\times$)          & 68.5          \\
      PSAGE \citep{zhang2022graphless}           & 67.4 (1.2$\times$)          & 47.3                      & 297.5 (1.1$\times$)          & 68.6          \\
      Neighbor Sample \citep{zhang2022graphless} & 25.5 (3.2$\times$)          & 47.0                      & 78.3 (4.0$\times$)           & 68.4          \\ \midrule\midrule
      \rowcolor{Gray}\textbf{SimMLP}             & \textbf{0.9 (89.7$\times$)} & \textbf{49.3}             & \textbf{2.5 (125.9$\times$)} & \textbf{70.2} \\
      \bottomrule
    \end{tabular}
  }

\end{table}

\begin{figure*}[!t]
  \centering
  \includegraphics[width=0.9\linewidth]{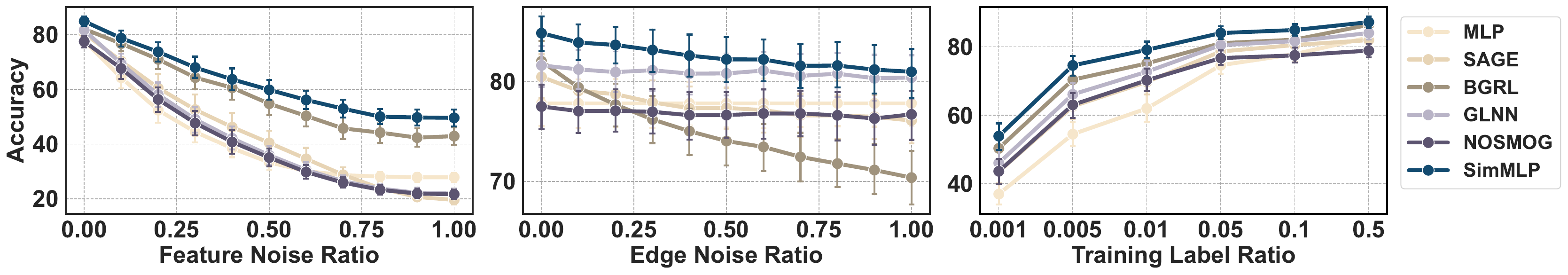}
  \caption{\textbf{Left}: Feature Noise. SimMLP is robustness to feature noise, whereas other MLP-based methods are susceptible to it. Middle: Edge Noise. SimMLP consistently demonstrates robustness against edge noise, even in scenarios with exceptionally high noise ratios. \textbf{Right}: Label Scarcity. SimMLP significantly outperforms baselines, especially with limited labels.}
  \label{fig:robustness}
\end{figure*}

\noindent\textbf{Inductive Setting.} In this paper, we consider inductive inference for unseen nodes within the same graph. We partition the graph $\gG = (\gV, E)$ into a non-overlapping transductive graph $\gG^T = (\gV^T, E^T)$ and an inductive graph $\gG^I = (\gV^I, E^I)$. The transductive graph $\gG^{T}$ contains 80\% of the nodes, further divided into $\gV^{T} = \gV^{T}{train} \sqcup \gV^{T}{valid} \sqcup \gV^{T}{test}$, which are used for training in the transductive setting. The inductive graph $\gG^{I}$ consists of the remaining 20\% of nodes, which are unseen during training. Compared to the settings in \citep{zhang2022graphless, tian2022learning}, our approach is more challenging because $\gV^I$ is disconnected from $\gV^T$ during inference. We evaluate three measures: (1) \textit{transductive} results, evaluated on $\gV^{T}{test}$, (2) \textit{inductive} results, evaluated on $\gV^I$, and (3) \textit{production} results, which are the weighted average of the previous two. The \textit{production} results are reported in Table \ref{tab:production}. We observe that SimMLP outperforms various baselines, highlighting the advantage of leveraging structural information for unseen nodes. Additional results can be found in Appendix \ref{sec:production full}.

\noindent\textbf{Cold-Start Setting.} We follow the inductive setting by partitioning the graph $\gG = (\gV, E)$ into a transductive graph $\gG^T$ and an inductive graph $\gG^I$. The key difference is that the nodes in $\gG^I$ are isolated, with all edges removed, such that $\gG^I = (\gV^I, \emptyset)$. This approach aligns with real-world applications where new users often emerge independently \citep{hao2021pre}. We report the performance on $\gV^I$ as the cold-start results, as shown in Table \ref{tab:cold-start}. SimMLP achieves significant improvements over all baselines, highlighting its superior structural learning capabilities. Notably, SimMLP shows performance gains of 7\% and 18\% over the vanilla MLP on Flickr and Arxiv, respectively, and 7\% and 6\% enhancements over VQGraph. Additional results can be found in Appendix \ref{sec:cold start full}.

\noindent\textbf{Heterophily Graphs.} The homophily inductive bias in SimMLP (Table \ref{tab:inductive bias}) stems from the chosen GNN encoder, which in our case is GCN. However, by employing different GNN architectures, we can introduce varying inductive biases that are better suited for heterophily graphs. To explore this, we use two advanced methods--ACMGCN \citep{luan2022revisiting} and GCNII \citep{chen2020simplegraph}--which are known for their effectiveness in handling heterophily graphs, alongside MLP and GCN. We evaluate these models on three challenging heterophily datasets: Actor, Texas, and Wisconsin. The results, presented in Figure \ref{fig:heterophily graph}, demonstrate that SimMLP can adapt to different heterophily-oriented GNN architectures, consistently enhancing their performance. This highlights SimMLP's adaptability and broad applicability across various graph domains and GNN architectures.

\subsection{Link Prediction}
\label{sec:link prediction}

Due to its self-supervised nature, SimMLP can be easily extended to link-level and graph-level tasks. We compare SimMLP against MLP, GNN, and LLP \citep{guo2023linkless}--a state-of-the-art MLP-based method for link prediction--using Cora, Citeseer, and Pubmed as benchmark datasets. We strictly adhere to the experimental setup from \citep{guo2023linkless}. The results, illustrated in Figure \ref{fig:link pred}, show that SimMLP achieves the best performance across all three datasets, with particularly significant improvements on Cora and Citeseer. These findings underscore SimMLP's superiority in modeling localized structural information for accurately determining node connectivity.

\subsection{Graph Classification}
\label{sec:graph classification}

Although SimMLP is primarily designed for learning localized structures, it still delivers strong performance on graph classification tasks, which require an understanding of global structural knowledge. We compare SimMLP against various baselines, including supervised GNNs such as GIN \citep{xu2018how}, as well as self-supervised GNNs like the WL kernel \citep{shervashidze2011weisfeiler}, DGK \citep{yanardag2015deep}, graph2vec \citep{narayanan2017graph2vec}, MVGRL \citep{hassani2020contrastive}, InfoGraph \citep{Sun2020InfoGraph}, GraphCL \citep{you2020graph}, and JOAO \citep{you2021graph}. Additionally, we implemented an MLP learning method on graphs by applying a KL divergence loss similar to that used in GLNN \citep{zhang2022graphless}. The graph classification results, presented in Table \ref{tab:graph classification}, show that SimMLP achieves the best or second-best performance on 6 out of 7 datasets across various domains, with particularly strong results on the large-scale COLLAB dataset. These findings highlight the potential of SimMLP for graph-level tasks.

\subsection{Inference Acceleration}

Table \ref{tab:inference acceleration} compares SimMLP with existing acceleration methods \citep{zhang2022graphless}, including quantization (QSAGE), pruning (PSAGE), and neighbor sampling (Neighbor Sample) in an inductive setting. We also include SGC and APPNP, which simplify message passing for faster inference. Our observations indicate that even the most efficient of these methods offers only marginal acceleration ($3.2\sim4.0\times$) and inevitably sacrifices model performance. In contrast, SimMLP achieves remarkable inference acceleration ($89.7\sim125.9\times$) by eliminating the neighborhood fetching process. Additionally, we provide a comparison with MLP-based methods, as shown in Figure \ref{fig:time_acc}. This figure illustrates the trade-off between model performance and inference time on Arxiv in a cold-start setting. SimMLP delivers significant performance improvements over MLP-based methods and substantial inference acceleration compared to GNN-based methods, demonstrating the best overall trade-off.

\subsection{Robustness Analysis}
\label{sec:robustness}

In this section, we analyze the robustness of SimMLP on noisy data and with scarce labels. We report the results for the inductive set $\gV^\gI$ in the inductive setting, averaging over seven datasets: Cora, Citeseer, Pubmed, Computer, Photo, Co-CS, and Wiki-CS.

\noindent\textbf{Noisy Node Features.} We assess the impact of node feature noise by introducing random Gaussian noise. Specifically, we replace $\mX$ with $\tilde\mX = (1 - \alpha)\mX + \alpha x$, where $x$ represents random noise independent of $\mX$, and the noise level $\alpha \in [0, 1]$. As shown in Figure \ref{fig:robustness} (\textbf{Left}), SimMLP outperforms all baselines across all settings, despite the fact that node content quality is critical for MLP-based methods \citep{zhang2022graphless,guo2023linkless}. We attribute this robustness to the augmentation process, which synthesizes additional high-quality node and ego-graph pairs, thereby enhancing MLP training. This augmentation also contributes to the robustness of BGRL. However, we observe that the performance of other MLP-based methods deteriorates rapidly as noise levels increase.

\noindent\textbf{Noisy Topology.} To introduce structural noise, we randomly flip edges within the graph. Specifically, we replace $\mA$ with $\tilde\mA = \mM \odot (1 - \mA) + (1 - \mM) \odot \mA$, where $\mM_{ij} \sim \gB(p)$, and $\gB(p)$ is a Bernoulli distribution with probability $p$. The results under varying noise levels are depicted in Figure \ref{fig:robustness} (\textbf{Middle}). SimMLP consistently outperforms other methods, demonstrating its robustness. While increasing noise levels significantly degrade the performance of GNNs, especially self-supervised BGRL, they have minimal impact on MLPs, even when $\tilde\mA$ becomes independent of $\mA$.

\noindent\textbf{Label Scarcity.} We also examine SimMLP's robustness under label scarcity. Figure \ref{fig:robustness} (\textbf{Right}) presents model performance across various label ratios for node classification. Our method consistently outperforms all other baselines, even with extremely limited training data (0.001). This highlights SimMLP's resilience to label scarcity. Furthermore, we observe that self-supervised methods demonstrate greater robustness \citep{huang2023towards} compared to supervised approaches, likely due to their ability to leverage unlabeled data during training.

\section{Conclusion}

MLPs offer rapid inference on graph-structured datasets, yet the lack in learning structural information limits their performance. In this paper, we propose SimMLP, the first MLP learning method that is equivalent to GNNs (in the optimal case). Our insight is that modeling the fine-grained correlation between node features and graph structures preserves generalized structural information. To instantiate the insight, we apply self-supervised alignment between GNNs and MLPs in the embedding space. Experimental results show that SimMLP is generalizable to unseen nodes, robust against noisy graphs and label scarcity, and flexible to various graph-related tasks.


\begin{acks}
  This work was partially supported by the NSF under grants IIS-2321504, IIS-2334193, IIS-2340346, IIS-2217239, CNS-2426514, CNS-2203261, and CMMI-2146076. Any opinions, findings, and conclusions or recommendations expressed in this material are those of the authors and do not necessarily reflect the views of the sponsors.
\end{acks}

\bibliographystyle{ACM-Reference-Format}
\bibliography{citation}

\newpage
\appendix

\section{Proof}
\label{sec:proof}

\subsection{Proof of Proposition \ref*{thm:objective}}
\label{proof:optimal encoder}

\begin{proof}
  Our proof is based on a mild assumption:
  \begin{itemize}
    \item The graph $\gG = (\mA, \mX)$ is sampled from a latent graph $\gG_\gI = (\mA, \mF)$ \citep{xie2022self} following the distribution $\gG \sim P(\gG_\gI)$, where $\mF \in \R^{N \times d}$ represents the latent node semantics. This assumption is the extension of latent variable assumption, a general assumption in statistics and machine learning, which is based on the principle that observed data is sampled from an unobserved distribution.
  \end{itemize}

  Then, we introduce some notations used in the proof:
  \begin{itemize}
    \item The MLP encoder $\gE$ and the decoder $\gD$ are defined as fully-connected layer, which ensures the $l$-Lipschitz continuity with respect to the $l_2$-norm. This is a common property in neural networks with continuous activation functions, e.g., ReLU \citep{virmaux2018lipschitz}. This property is data-agnostic and thus suitable for most real-world graphs. The output of MLP encoder is $\mH^{MLP} = \gE(\mX)$ with $\mH^{MLP} \in \R^{N \times d'}$.
    \item GNN encoder $\phi$ yields $\mH^{GNN} = \phi(\mX, \mA)$ with $\mH^{GNN} \in \R^{N \times d'}$.
    \item $\mF^* \in \R^{N \times d'}$ denotes the lossless compression of $\mF$ that $\E[\mX | \mA, \mF^*] = \mF$.
  \end{itemize}

  The Equation \ref{eq:thm objective} can be rewritten as:

  \begin{align}
    \nonumber\gE^* = & \argmin_{\gE} \E {\left\| \mH^{MLP} - \mH^{GNN} \right\|}^{2} +  \left\|\gD(\mH^{GNN}) - \mX \right\|^2                                                \\
    \nonumber=       & \argmin_{\gE} \E {\left\|( \mH^{MLP} - \mF^*) - (\mH^{GNN} - \mF^*) \right\|}^{2} +  \left\|\gD(\mH^{GNN}) - \mX \right\|^2                            \\
    \nonumber=       & \argmin_{\gE} \E \left[ \left\| \mH^{MLP} - \mF^* \right\|^2 + \left\| \mH^{GNN} - \mF^* \right\|^2 + \left\| \gD(\mH^{GNN}) - \mX \right\|^2  \right] \\
    \nonumber        & - 2  \E \left[ \left\langle \mH^{MLP} - \mF^*, \mH^{GNN} - \mF^* \right\rangle   \right]                                                               \\
    \nonumber=       & \argmin_{\gE} \E \left[ \left\| \mH^{MLP} - \mF^* \right\|^2 + \left\| \mH^{GNN} - \mF^* \right\|^2 + \left\| \gD(\mH^{GNN}) - \mX \right\|^2  \right] \\
    \nonumber        & - 2  \E_{\mF^*} \left[ \sum_i \left( \mH_i^{MLP} - \mF^*_i \right) \left( \mH_i^{GNN} - \mF^*_i \right) | \mF^* \right]                                \\
    \nonumber=       & \argmin_{\gE} \E \left[ \left\| \mH^{MLP} - \mF^* \right\|^2 + \left\| \mH^{GNN} - \mF^* \right\|^2 + \left\| \gD(\mH^{GNN}) - \mX \right\|^2  \right] \\
    \nonumber        & - 2 \E_{\mF^*} \left[ \sum_i \Cov(\mH^{MLP}_i - \mF^*, \mH^{GNN}_i - \mF^*) | \mF^* \right].                                                           \\
    \nonumber=       & \argmin_{\gE} \E \left[ \left\| \mH^{MLP} - \mF^* \right\|^2 + \left\| \mH^{GNN} - \mF^* \right\|^2 + \left\| \gD(\mH^{GNN}) - \mX \right\|^2  \right] \\
                     & - 2 \E_{\mF^*} \left[ \sum_i \Cov(\mH^{MLP}_i, \mH^{GNN}_i) | \mF^* \right].
    \label{eq:extended}
  \end{align}

  Then, with a bit of simple transformations, the Equation \ref{eq:extended} can be expressed in the form of Equation \ref{eq:thm objective}. We explain these four terms in details. The first two terms $\left\| \mH^{MLP} - \mF^* \right\|^2$ and $\left\| \mH^{GNN} - \mF^* \right\|^2$ indicate the reconstruction errors of MLP embedding $\mH^{MLP}$ and GNN embedding $\mH^{GNN}$ on the latent variable $\mF^*$, ensuring the invariance on the latent graph $\gG_\gI$. The third term $\left\| \gD(\mH^{GNN}) - \mX \right\|^2$ reconstructs the node feature $\mX$ using GNN embeddings $\mH^{GNN}$, mitigating the risk of potential distribution shifts. The last term $-2 \sum_i \Cov(\mH^{MLP}_i, \mH^{GNN}_i)$ maximizes the covariance between GNN and MLP embeddings at each dimension, aligning GNNs and MLPs in the embedding space.

\end{proof}

\subsection{Proof of Lemma \ref*{thm:CE MI}}
\label{proof:CE MI}

\begin{proof}

  We follow the paper \citep{boudiaf2020unifying} to prove the lemma. We show the equivalence between these two terms by expanding $H(\mY; \hat{\mY} | \mX)$ and $I(\mX; \mY)$. We first expand the mutual information as
  \begin{equation}
    I(\mX; \mY) = H(\mY) - H(\mY | \mX).
  \end{equation}
  Maximizing the mutual information $I(\mX; \mY)$ indicates minimizing the conditional entropy $H(\mY | \mX)$. The entropy on the label $H(\mY)$ is a constant, which can be ignored.

  The cross-entropy $H(\mY; \hat{\mY} | \mX)$ can be written as the combination of conditional entropy $H(\mY | \mX)$ and KL divergence $\gD_{KL}( \mY \| \hat{\mY} | \mX )$:
  \begin{align}
    \nonumber H(\mY; \hat{\mY} | \mX) & = -\sum_i (\mY_i | \mX) \log (\hat{\mY}_i | \mX)                                     \\
    \nonumber                         & = -\sum_i (\mY_i | \mX) \log (\mY_i | \mX) + \sum_i (\mY_i | \mX) \log (\mY_i | \mX) \\
    \nonumber                         & -\sum_i \mY_i \log (\hat{\mY}_i | \mX)                                               \\
    \nonumber                         & = H(\mY | \mX) + \sum_i (\mY_i | \mX) \log \frac{(\mY_i | \mX)}{(\hat{\mY}_i | \mX)} \\
                                      & = H(\mY | \mX) + \gD_{KL}(\mY \| \hat{\mY} | \mX)
    \label{eq:CE MI}
  \end{align}

  Considering Equation \ref{eq:CE MI}, minimizing the cross-entropy $H(\mY; \hat{\mY} | \mX)$ can minimize $H(\mY | \mX)$ (as well as $\gD_{KL}( \mY \| \hat{\mY} | \mX )$), which is equivalent to maximizing the mutual information $I(\mX; \mY)$. Based on the analysis in \citep{boudiaf2020unifying}, Equation \ref{eq:CE MI} can be optimized in a \textit{Max-Min} manner. In particular, the first step is to freeze the encoder and only optimize the classifier, corresponding to fix $H(\mY | \mX)$ and minimize $\gD_{KL}(\mY \| \hat{\mY} | \mX)$. The KL term would ideally vanish at the end of this step. Following step involves optimizing the parameters of the encoder while fixing the classifier.

\end{proof}

\subsection{Proof of Proposition \ref*{thm:information bottleneck}}
\label{proof:optimal compression}

\noindent\textit{Proof.} Before the proof, we need to provide some notations. We aim to compress the original graph $\gG$ into $\mT = (\mH^{MLP}, \mH^{GNN})$ by preserving the information of latent graph $\gG_\gI$. Based on the definition of information bottleneck \citep{tishby2000information}, the optimal compression is
\begin{equation}
  \mT^* = \argmin_\mT I(\gG; \mT) - \beta I(\mT; \gG_\gI),
\end{equation}
where $\beta$ denotes the Lagrange multiplier and $I(\cdot ; \cdot)$ is the mutual information. The optimal compression $\mT^*$ preserves the essential latent information by maximizing $\beta I(\mT; \gG_\gI)$ and discard the noises contained in the observed data $\gG$ by minimizing $I(\gG; \mT)$. To handle the equation in a more accessible manner, we convert it as
\begin{align}
  \nonumber \mT^* & = \argmin_\mT I(\gG; \mT) - \beta I(\mT; \gG_\gI)                           \\
  \nonumber       & = \argmin_\mT (1 - \beta) H(\mT) + \beta H(\mT | \gG_\gI) - H(\mT | \gG)    \\
  \nonumber       & = \argmin_\mT H(\mT) + \lambda H(\mT |  \gG_\gI)                            \\
  \nonumber       & = \argmin_{\mH^{MLP}, \mH^{GNN}} \lambda H(\mH^{MLP} | \gG_\gI)             \\
                  & + H(\mH^{GNN}) + \lambda H(\mH^{GNN} | \gG_\gI) + H(\mH^{MLP} | \mH^{GNN}),
\end{align}
where $\lambda = \frac{\beta}{1 - \beta} > 0$ and $H(\cdot)$ is the entropy.

\section{Experimental Setup}
\label{sec:appendix exp setup}


\subsection{Dataset Statistics}
\label{sec:dataset}

\begin{table}[!t]
  \centering
  \caption{The statistics of node classification datasets. }
  \resizebox{\linewidth}{!}{
    \begin{tabular}{lcccc}
      \toprule
      \textbf{Dataset} & \textbf{\# Nodes} & \textbf{\# Edges} & \textbf{\# Features} & \textbf{\# Classes} \\ \midrule\midrule
      Cora             & 2,708             & 10,556            & 1,433                & 7                   \\
      Citeseer         & 3,327             & 9,104             & 3,703                & 6                   \\
      Pubmed           & 19,717            & 88,648            & 500                  & 3                   \\
      Computer         & 13,752            & 491,722           & 767                  & 10                  \\
      Photo            & 7,650             & 238,162           & 745                  & 8                   \\
      Co-CS            & 18,333            & 163,788           & 6,805                & 15                  \\
      Co-Phys          & 34,493            & 495,924           & 8,415                & 5                   \\
      Wiki-CS          & 11,701            & 432,246           & 300                  & 10                  \\
      Flickr           & 89,250            & 899,756           & 500                  & 7                   \\
      Arxiv            & 169,343           & 1,166,243         & 128                  & 40                  \\ \midrule\midrule
      Texas            & 183               & 295               & 1,703                & 5                   \\
      Wisconsin        & 251               & 515               & 1,703                & 5                   \\
      Actor            & 7,600             & 30,019            & 932                  & 5                   \\
      \bottomrule
    \end{tabular}
  }
  \label{tab:dataset}
\end{table}

\begin{table}[!t]
  \centering
  \caption{The statistics of graph classification datasets. }
  \resizebox{\linewidth}{!}{
    \begin{tabular}{lccccc}
      \toprule
      \textbf{Dataset} & \textbf{\# Graphs} & \textbf{\# Nodes} & \textbf{\# Edges} & \textbf{\# Features} & \textbf{\# Classes} \\ \midrule\midrule
      IMDB-B           & 1,000              & $\sim$19.8        & $\sim$193.1       & -                    & 2                   \\
      IMDB-M           & 1,500              & $\sim$13.0        & $\sim$65.9        & -                    & 3                   \\
      COLLAB           & 5,000              & $\sim$74.5        & $\sim$4,914.4     & -                    & 3                   \\
      PTC-MR           & 344                & $\sim$14.3        & $\sim$14.7        & 18                   & 2                   \\
      MUTAG            & 118                & $\sim$17.9        & $\sim$39.6        & 7                    & 2                   \\
      DD               & 1,178              & $\sim$284.3       & $\sim$715.6       & 89                   & 2                   \\
      PROTEINS         & 1,113              & $\sim$39.1        & $\sim$145.6       & 3                    & 2                   \\ \bottomrule
    \end{tabular}
  }
  \label{tab:statistics graph clf}
\end{table}

\subsubsection*{Node Classification.} We select 10 benchmark datasets to evaluate the performance of SimMLP and other baselines. These datasets are collected from diverse domains, encompassing citation networks, social networks, wikipedia networks, etc. We present the statistics of these datasets in Table \ref{tab:dataset}. Specifically, Cora, Citeseer, Pubmed \citep{yang2016revisiting} are three citation networks, in which nodes denote papers and edges represent citations. The node features are represented as bag-of-words based on paper keywords. Computer (Amazon-CS) and Photo (Amazon-Photo) \citep{shchur2018pitfalls} are two co-purchase networks that describe the frequent co-purchases of items (nodes). Co-CS (Coauthor-CS) and Co-Phys (Coauthor-physics) \citep{shchur2018pitfalls} consist of nodes representing authors and edges indicating collaborations between authors. Wiki-CS \citep{mernyei2020wiki} is extracted from Wikipedia, comprising computer science articles (nodes) connected by hyperlinks (edges). Flickr \citep{Zeng2020GraphSAINT} consists online images, with the goal of categorizing images based on their descriptions and common properties. All these datasets are available through PyG (Pytorch Geometric), and we partition them randomly into training, validation, and testing sets with a split ratio of 10\%/10\%/80\%. Additionally, we employ Arxiv dataset from OGB benchmarks \citep{hu2020open} to evaluate model performance on large-scale datasets. We process the dataset in PyG using OGB public interfaces with standard public split setting.

\subsubsection*{Link Prediction.} We take five public benchmark datasets, including Cora, Citeseer, Pubmed, Computer, and Photo, for evaluating link prediction performance. The statistics is presented in Table \ref{tab:dataset}, and we adopt the hyper-parameters in Table \ref{tab:hyperparameter node}.

\subsubsection*{Graph Classification.} For graph classification, all datasets are sourced from TU datasets \citep{morris2020tudataset}\footnote{These datasets are available in PyG library.}, including biochemical molecule datasets (PTC-MR, MUTAG, DD, PROTEINS) and social networks (IMDB-B, IMDB-M, COLLAB). Table \ref{tab:statistics graph clf} shows the statistics of these datasets. In PTC-MR and DD, we utilize the original node features, whereas for other datasets lacking rich node features, we generate one-hot features based on node degrees. We follow 10-fold cross-validation to evaluate model performance.

\subsection{Summary of Baselines}
\label{sec:baselines}

We compare SimMLP against a range of baselines, encompassing supervised GNNs, self-supervised graph contrastive learning (GCL) methods, and MLP-based graph learning methods.

\subsubsection*{Node Classification.}

\textbf{Supervised GNNs}: Our primary node classification baselines include GraphSAGE \citep{hamilton2017inductive} and GAT \citep{velickovic2018graph}. Furthermore, we also incorporate SGC \citep{wu2019simplifying} and APPNP \citep{gasteiger2018combining} as additional node classification baselines. \textbf{Self-supervised GNNs}: We compare SimMLP to self-supervised graph learning methods. DGI \citep{velickovic2018deep} and MVGRL \citep{hassani2020contrastive} conduct contrastive learning between graph patches and graph summaries to integrate knowledge into node embeddings. GRACE \citep{zhu2020deep} and subsequent GCA \citep{zhu2021graph} perform contrast between nodes in two corrupted views to acquire augmentation-invariant embeddings. BGRL \citep{thakoor2022largescale} utilizes predictive objective for node-level contrastive learning to achieve efficient training. \textbf{MLPs on Graphs}: In node classification, we employ basic MLP that considers only node content as baseline. Furthermore, we incorporate GraphMLP \citep{hu2021graph} that trains an MLP by emphasizing consistency between target nodes and their direct neighborhoods. We exclude the following works \citep{dong2022node,liu2022mlp} as baselines since they are high-order versions of GraphMLP. To achieve this, we slightly modify the original GraphMLP to enable the ability in learning high-order information, and search the number of layers within \{1, 2, 3\}. GLNN \citep{zhang2022graphless} employs knowledge distillation to transfer knowledge from GNNs to MLPs, GENN leverages positional encoding to acquire structural knowledge, while NOSMOG \citep{tian2022learning} jointly integrates positional information and robust training strategies based on GLNN. Note that the public code of GENN is not available, thus we implement GENN based on the code of NOSMOG. VQGraph \citep{yang2023vqgraph} is the recent SOTA method that leverages codebook to learn structural information.

\subsubsection*{Link Prediction}

For baselines, we use the basic GNN, MLP, and LLP \citep{guo2023linkless}, a state-of-the-art MLP learning framework for link prediction. We strictly follow the experimental settings of \citet{guo2023linkless} and adopt the AUC as metric.

\subsubsection*{Graph Classification}

We use the following baselines. \textbf{Supervised GNNs}: We utilize 5-layer GIN \citep{xu2018how} as the baseline. \textbf{Self-supervised GNNs}: For graph-level tasks, we explore traditional graph kernels for classification, including WL kernel \citep{shervashidze2011weisfeiler} and DGK \citep{yanardag2015deep}. Furthermore, we include advanced contrastive learning approaches, such as graph2vec \citep{narayanan2017graph2vec}, MVGRL \citep{hassani2020contrastive}, InfoGraph \citep{Sun2020InfoGraph}, GraphCL \citep{you2020graph}, and JOAO \citep{you2021graph}, which conduct contrastive learning between embeddings of two augmented graphs. \textbf{MLPs on Graphs}: For standard MLP, we append a pooling function following the encoder to generate graph embeddings, which are utilized to perform predictions. Considering other MLP learning baselines, they cannot be directly applied on graph-level tasks. To this end, we extend GLNN \citep{zhang2022graphless} to graph classification by distilling knowledge from pre-trained GINs to MLPs on graph-level embeddings, dubbed as MLP + KD.

\subsection{Hyper-parameter setting}
\label{sec:hyperparameter}

\begin{table*}[!ht]
  \centering
  \caption{Hyper-parameters used for SimMLP for node-level task. }
  \resizebox{0.8\linewidth}{!}{
    \begin{tabular}{lcccccccccc}
      \toprule
                         & \multicolumn{10}{c}{Node Classification              \& Link Prediction}                                                                                     \\ \cmidrule(lr){2-11}
      Hyper-parameters   & Cora                                                                     & Citeseer & Pubmed & Computer & Photo & Co-CS & Co-Phys & Wiki-CS & Flickr & Arxiv \\ \midrule\midrule
      Epochs             & 1000                                                                     & 1000     & 1000   & 1000     & 1000  & 2000  & 1000    & 2000    & 2000   & 5000  \\
      Optimizer          & \multicolumn{10}{c}{AdamW used for all datasets}                                                                                                             \\
      Learning Rate      & 1e-3                                                                     & 5e-4     & 5e-4   & 1e-3     & 1e-3  & 1e-4  & 1e-3    & 5e-4    & 1e-3   & 1e-3  \\
      Weight Decay       & 0                                                                        & 5e-5     & 1e-5   & -        & 1e-4  & -     & 1e-4    & 1e-5    & 5e-4   & -     \\
      Activation         & \multicolumn{10}{c}{PReLU used for all datasets}                                                                                                             \\
      Hidden Dimension   & 512                                                                      & 512      & 512    & 512      & 512   & 512   & 512     & 512     & 1024   & 1024  \\
      Normalization      & \multicolumn{10}{c}{Batchnorm used for all datasets}                                                                                                         \\
      \# MLP Layers      & 2                                                                        & 2        & 2      & 3        & 2     & 2     & 2       & 2       & 2      & 8     \\
      \# GNN Layers      & 2                                                                        & 3        & 3      & 2        & 1     & 1     & 1       & 2       & 3      & 3     \\
      Feature Mask Ratio & 0.50                                                                     & 0.75     & 0.25   & 0.25     & 0.25  & 0.50  & 0.75    & 0.00    & 0.25   & 0.00  \\
      Edge Mask Ratio    & 0.25                                                                     & 0.50     & 0.25   & 0.25     & 0.50  & 0.75  & 0.50    & 0.25    & 0.50   & 0.25  \\ \bottomrule
    \end{tabular}
  }
  \label{tab:hyperparameter node}
\end{table*}

\begin{table*}[!ht]
  \centering
  \caption{Hyper-parameters of SimMLP on graph-level task. }
  \resizebox{0.7\linewidth}{!}{
    \begin{tabular}{lccccccc}
      \toprule
                           & \multicolumn{7}{c}{\textbf{Graph Classification}}                                                        \\ \cmidrule(lr){2-8}
      Hyper-parameters     & IMDB-B                                              & IMDB-M & COLLAB & PTC-MR & MUTAG & DD   & PROTEINS \\ \midrule\midrule
      Epochs               & 200                                                 & 100    & 30     & 100    & 100   & 100  & 500      \\
      Optimizer            & \multicolumn{7}{c}{AdamW used for all datasets}                                                          \\
      Learning Rate        & 1e-2                                                & 1e-2   & 5e-4   & 1e-2   & 1e-2  & 1e-3 & 1e-3     \\
      Weight Decay         & \multicolumn{7}{c}{0 used for all datasets}                                                              \\
      Batch Size           & 64                                                  & 128    & 32     & 64     & 64    & 32   & 64       \\
      Hidden Dimension     & \multicolumn{7}{c}{512 used for all datasets}                                                            \\
      Pooling              & MEAN                                                & MEAN   & MEAN   & SUM    & SUM   & MEAN & SUM      \\
      Activation           & \multicolumn{7}{c}{PReLU used for all datasets}                                                          \\
      Normalization        & \multicolumn{7}{c}{Batchnorm used for all datasets}                                                      \\
      Raw Feature          & N                                                   & N      & N      & Y      & N     & Y    & N        \\
      Deg4Feature          & Y                                                   & Y      & Y      & N      & Y     & N    & Y        \\
      \# Encoder Layers    & \multicolumn{7}{c}{2 used for all datasets}                                                              \\
      \# Aggregator Layers & 2                                                   & 2      & 2      & 2      & 1     & 2    & 1        \\
      Feature Mask Ratio   & 0.50                                                & 0.25   & 0.75   & 0.25   & 0.5   & 0.00 & 0.00     \\
      Edge Mask Ratio      & 0.75                                                & 0.50   & 0.75   & 0.00   & 0.25  & 0.00 & 0.50     \\ \bottomrule
    \end{tabular}
  }
  \label{tab:hyperparameter graph}
\end{table*}

We run each experiment 10 times with different seeds to alleviate the impact of randomness. We perform hyper-parameter tuning for each approach using a grid search strategy. Specifically, we set the number of epochs to 1,000, the hidden dimension to 512, and employ PReLU as the activation function. We explore various learning rates \{5e-4, 1e-4, 5e-4, 1e-3, 5e-3, 1e-2\}, weight decay values \{5e-5, 1e-5, 5e-3, 1e-4, 0\}, and the number of layers \{1, 2, 3\}. In self-supervised learning methods, we employ a 2-layer GCN \citep{kipf2017semisupervised} as the encoder for node-level tasks. Subsequently, we assess the quality of the acquired embeddings by training a Logistic regression function on downstream tasks \citep{zhu2020deep}. For other settings, we follow the settings reported in the original papers. Regarding SimMLP, we provide a comprehensive overview of the hyper-parameter settings for node classification task in Table \ref{tab:hyperparameter node}. For graph classification, the experimental setting is the same as Sec. \ref{sec:hyperparameter}. The only exclusion is we utilize a 5-layer GIN \citep{xu2018how} as the encoder. The readout function are selected from \{MEAN, SUM, MAX\}. The hyper-parameters of SimMLP is in Table \ref{tab:hyperparameter graph}.

\section{Additional Results}

\subsection{Full Inductive Setting Results}
\label{sec:production full}

See Table \ref{tab:ind full}.

\begin{table*}[!ht]
  \centering
  \caption{Node classification accuracy (\%) under inductive (production) scenario for both transductive and inductive settings. {\it ind} represents the accuracy on $\mathcal{V}^{I}$, {\it trans} represents the accuracy on $\mathcal{V}_{test}^T$, and {\it prod} is the interpolated accuracy of both {\it ind} and {\it trans}. }
  \label{tab:ind full}
  \resizebox{\linewidth}{!}{
    \begin{tabular}{clcccccccccc|c}
      \toprule
      Methods                   & Setting        & Cora              & Citeseer          & Pubmed            & Computer          & Photo             & Co-CS             & Co-Phys           & Wiki-CS           & Flickr            & OGB-Arxiv         & \textit{\textbf{Avg.}} \\ \midrule
                                & \textit{prod}  & 77.5{\small ±1.8} & 68.4{\small ±1.6} & 85.0{\small ±0.4} & 87.2{\small ±0.4} & 93.2{\small ±0.5} & 92.9{\small ±0.4} & 95.7{\small ±0.1} & 79.3{\small ±0.7} & 47.2{\small ±0.7} & 68.5{\small ±0.6} & 79.5                   \\
                                & \textit{trans} & 79.5{\small ±1.5} & 68.7{\small ±1.4} & 85.6{\small ±0.3} & 88.0{\small ±0.3} & 93.7{\small ±0.4} & 93.1{\small ±0.3} & 95.8{\small ±0.0} & 80.0{\small ±0.4} & 48.2{\small ±0.6} & 71.8{\small ±0.5} & 80.4                   \\
      \multirow{-3}{*}{SAGE}    & \textit{ind}   & 69.7{\small ±2.9} & 67.1{\small ±2.6} & 82.9{\small ±1.0} & 84.5{\small ±0.9} & 91.2{\small ±0.6} & 91.9{\small ±0.7} & 95.6{\small ±0.1} & 76.3{\small ±1.6} & 43.3{\small ±1.1} & 55.5{\small ±0.8} & 75.8                   \\ \midrule
                                & \textit{prod}  & 77.7{\small ±1.1} & 64.3{\small ±1.6} & 84.0{\small ±0.5} & 87.3{\small ±0.5} & 91.5{\small ±0.6} & 91.3{\small ±0.4} & 94.4{\small ±0.3} & 76.3{\small ±1.1} & 49.1{\small ±0.3} & 69.3{\small ±0.4} & 78.5                   \\
                                & \textit{trans} & 77.3{\small ±0.9} & 64.2{\small ±1.4} & 84.0{\small ±0.3} & 87.3{\small ±0.4} & 91.5{\small ±0.5} & 91.3{\small ±0.3} & 94.4{\small ±0.3} & 76.3{\small ±1.0} & 49.1{\small ±0.2} & 70.4{\small ±0.4} & 78.6                   \\
      \multirow{-3}{*}{BGRL}    & \textit{ind}   & 79.4{\small ±1.7} & 65.0{\small ±2.2} & 84.0{\small ±1.0} & 87.6{\small ±0.8} & 91.5{\small ±1.1} & 91.1{\small ±0.5} & 94.3{\small ±0.5} & 76.0{\small ±1.6} & 49.3{\small ±0.6} & 65.0{\small ±0.5} & 78.3                   \\ \midrule
                                & \textit{prod}  & 63.8{\small ±1.7} & 64.0{\small ±1.2} & 80.9{\small ±0.5} & 81.0{\small ±0.5} & 87.7{\small ±0.9} & 91.7{\small ±0.6} & 95.2{\small ±0.1} & 75.1{\small ±0.7} & 46.1{\small ±0.2} & 55.9{\small ±0.5} & 74.1                   \\
                                & \textit{trans} & 63.7{\small ±1.5} & 63.9{\small ±1.1} & 80.9{\small ±0.4} & 81.1{\small ±0.5} & 87.7{\small ±0.9} & 91.7{\small ±0.5} & 95.2{\small ±0.1} & 75.1{\small ±0.4} & 46.2{\small ±0.2} & 55.9{\small ±0.5} & 74.1                   \\
      \multirow{-3}{*}{MLP}     & \textit{ind}   & 64.2{\small ±2.1} & 64.4{\small ±1.8} & 80.9{\small ±0.7} & 80.8{\small ±0.9} & 87.9{\small ±1.0} & 91.8{\small ±0.8} & 95.2{\small ±0.2} & 74.9{\small ±1.8} & 46.1{\small ±0.5} & 55.9{\small ±0.7} & 74.2                   \\ \midrule
                                & \textit{prod}  & 78.3{\small ±1.0} & 69.6{\small ±1.1} & 85.4{\small ±0.5} & 87.0{\small ±0.5} & 93.3{\small ±0.4} & 93.7{\small ±0.4} & 95.8{\small ±0.1} & 78.4{\small ±0.5} & 46.1{\small ±0.3} & 63.5{\small ±0.5} & 79.1                   \\
                                & \textit{trans} & 79.9{\small ±0.9} & 69.7{\small ±0.8} & 85.7{\small ±0.4} & 87.8{\small ±0.5} & 93.8{\small ±0.4} & 93.8{\small ±0.3} & 95.8{\small ±0.0} & 78.6{\small ±0.3} & 46.1{\small ±0.2} & 64.3{\small ±0.5} & 79.6                   \\
      \multirow{-3}{*}{GLNN}    & \textit{ind}   & 72.0{\small ±1.7} & 69.1{\small ±2.6} & 84.4{\small ±0.9} & 84.0{\small ±0.7} & 91.1{\small ±0.5} & 93.3{\small ±0.5} & 95.7{\small ±0.1} & 77.6{\small ±1.4} & 46.1{\small ±0.4} & 60.6{\small ±0.6} & 77.4                   \\ \midrule
                                & \textit{prod}  & 77.8{\small ±1.6} & 67.3{\small ±1.5} & 84.3{\small ±0.5} & 85.8{\small ±1.2} & 92.1{\small ±1.0} & 93.6{\small ±0.4} & 95.7{\small ±0.1} & 78.3{\small ±1.0} & 45.6{\small ±0.5} & 68.5{\small ±0.5} & 78.9                   \\
                                & \textit{trans} & 80.3{\small ±1.4} & 67.9{\small ±1.2} & 85.8{\small ±0.4} & 87.4{\small ±1.0} & 93.4{\small ±0.6} & 93.8{\small ±0.4} & 95.8{\small ±0.1} & 80.3{\small ±0.9} & 45.7{\small ±0.5} & 70.0{\small ±0.5} & 80.0                   \\
      \multirow{-3}{*}{GENN}    & \textit{ind}   & 68.1{\small ±2.2} & 65.1{\small ±2.8} & 78.4{\small ±0.8} & 79.1{\small ±1.8} & 87.1{\small ±2.4} & 92.7{\small ±0.5} & 95.2{\small ±0.1} & 70.1{\small ±1.7} & 45.1{\small ±0.7} & 62.6{\small ±0.7} & 74.3                   \\ \midrule
                                & \textit{prod}  & 78.4{\small ±1.8} & 70.4{\small ±1.1} & 85.4{\small ±0.6} & 87.4{\small ±1.0} & 93.3{\small ±0.7} & 93.7{\small ±0.4} & 95.8{\small ±0.1} & 79.0{\small ±1.0} & 46.4{\small ±0.4} & 69.3{\small ±0.9} & 79.9                   \\
                                & \textit{trans} & 80.4{\small ±1.4} & 70.4{\small ±1.0} & 85.6{\small ±0.3} & 88.2{\small ±0.9} & 93.8{\small ±0.4} & 93.9{\small ±0.4} & 95.8{\small ±0.1} & 79.4{\small ±0.9} & 46.4{\small ±0.3} & 70.6{\small ±0.7} & 80.4                   \\
      \multirow{-3}{*}{VQGraph} & \textit{ind}   & 70.4{\small ±3.4} & 70.0{\small ±1.6} & 84.5{\small ±1.5} & 84.3{\small ±1.1} & 91.5{\small ±1.8} & 93.0{\small ±0.6} & 95.7{\small ±0.3} & 77.5{\small ±1.7} & 46.3{\small ±0.9} & 64.0{\small ±1.7} & 77.7                   \\ \midrule
                                & \textit{prod}  & 77.8{\small ±1.9} & 68.6{\small ±1.4} & 83.8{\small ±0.5} & 86.6{\small ±1.2} & 92.5{\small ±0.7} & 93.5{\small ±0.4} & 95.8{\small ±0.1} & 78.4{\small ±0.7} & 46.1{\small ±0.6} & 69.1{\small ±0.8} & 79.2                   \\
                                & \textit{trans} & 80.3{\small ±1.7} & 69.0{\small ±1.2} & 85.4{\small ±0.4} & 88.3{\small ±1.1} & 93.9{\small ±0.5} & 93.7{\small ±0.4} & 95.9{\small ±0.1} & 80.4{\small ±0.6} & 46.2{\small ±0.5} & 70.5{\small ±0.8} & 80.4                   \\
      \multirow{-3}{*}{NOSMOG}  & \textit{ind}   & 68.1{\small ±3.0} & 67.1{\small ±2.1} & 77.4{\small ±0.8} & 79.8{\small ±1.5} & 87.1{\small ±1.5} & 92.6{\small ±0.7} & 95.5{\small ±0.1} & 70.4{\small ±1.2} & 45.3{\small ±0.7} & 63.5{\small ±0.8} & 74.7                   \\ \midrule
                                & \textit{prod}  & 81.4{\small ±1.2} & 72.3{\small ±0.9} & 86.5{\small ±0.3} & 87.7{\small ±0.4} & 93.9{\small ±0.3} & 94.6{\small ±0.2} & 96.0{\small ±0.1} & 79.3{\small ±0.8} & 49.3{\small ±0.2} & 70.2{\small ±0.5} & 81.1                   \\
                                & \textit{trans} & 81.6{\small ±1.0} & 72.2{\small ±0.7} & 86.5{\small ±0.2} & 87.7{\small ±0.3} & 93.9{\small ±0.3} & 94.7{\small ±0.2} & 96.1{\small ±0.1} & 79.5{\small ±0.7} & 49.2{\small ±0.1} & 71.3{\small ±0.3} & 81.3                   \\
      \multirow{-3}{*}{SimMLP}  & \textit{ind}   & 80.5{\small ±2.2} & 72.8{\small ±1.6} & 86.4{\small ±0.5} & 87.6{\small ±1.0} & 93.9{\small ±0.6} & 94.5{\small ±0.2} & 96.0{\small ±0.2} & 78.5{\small ±1.5} & 49.4{\small ±0.5} & 66.1{\small ±1.1} & 80.6                   \\
      \bottomrule
    \end{tabular}
  }
\end{table*}

\subsection{Full Cold Start Results}
\label{sec:cold start full}

See Table \ref{tab:cold-start full}.

\begin{table*}[!t]
  \centering
  \caption{Node classification accuracy under cold-start setting. }
  \label{tab:cold-start full}

  \resizebox{\linewidth}{!}{
    \begin{tabular}{lcccccccccc | c}
      \toprule
      Methods                            & Cora                          & Citeseer                      & Pubmed                        & Computer                      & Photo                         & Co-CS                         & Co-Phys                       & Wiki-CS                       & Flickr                        & OGB-Arxiv                     & \textbf{\textit{Avg.}} \\ \midrule\midrule
      SAGE \citep{hamilton2017inductive} & 69.7{\small ±2.9}             & \underline{67.1{\small ±2.6}} & 82.9{\small ±1.0}             & 84.5{\small ±0.9}             & 91.2{\small ±0.6}             & \underline{91.9{\small ±0.7}} & \underline{95.6{\small ±0.1}} & \underline{76.3{\small ±1.6}} & 43.3{\small ±1.1}             & 55.5{\small ±0.8}             & 75.8                   \\
      BGRL \citep{thakoor2022largescale} & \underline{79.4{\small ±1.7}} & 65.0{\small ±2.2}             & \underline{84.0{\small ±1.0}} & \textbf{87.6{\small ±0.8}}    & \underline{91.5{\small ±1.1}} & 91.1{\small ±0.5}             & 94.3{\small ±0.5}             & 76.0{\small ±1.6}             & \underline{49.3{\small ±0.6}} & \underline{65.0{\small ±0.5}} & \underline{78.3}       \\ \midrule\midrule
      MLP \citep{zhang2022graphless}     & 64.2{\small ±2.1}             & 64.4{\small ±1.8}             & 80.9{\small ±0.7}             & 80.8{\small ±0.9}             & 87.9{\small ±1.0}             & 91.8{\small ±0.8}             & 95.2{\small ±0.2}             & 74.9{\small ±1.8}             & 46.1{\small ±0.5}             & 55.9{\small ±0.7}             & 74.2                   \\
      GLNN \citep{zhang2022graphless}    & \underline{72.0{\small ±1.7}} & 69.1{\small ±2.6}             & 84.4{\small ±0.9}             & 84.0{\small ±0.7}             & 91.1{\small ±0.5}             & \underline{93.3{\small ±0.5}} & \underline{95.7{\small ±0.1}} & \underline{77.6{\small ±1.4}} & 46.1{\small ±0.4}             & 60.6{\small ±0.6}             & 77.4                   \\
      GENN \citep{wang2023graph}         & 68.1{\small ±2.2}             & 65.1{\small ±2.8}             & 78.4{\small ±0.8}             & 79.1{\small ±1.8}             & 87.1{\small ±2.4}             & 92.7{\small ±0.5}             & 95.2{\small ±0.1}             & 70.1{\small ±1.7}             & 45.1{\small ±0.7}             & 62.6{\small ±0.7}             & 74.3                   \\
      VQGraph \citep{yang2023vqgraph}    & 70.4{\small ±3.4}             & \underline{70.0{\small ±1.6}} & \underline{84.5{\small ±1.5}} & \underline{84.3{\small ±1.1}} & \underline{91.5{\small ±1.8}} & 93.0{\small ±0.6}             & 95.7{\small ±0.3}             & 77.5{\small ±1.7}             & \underline{46.3{\small ±0.9}} & \underline{64.0{\small ±1.7}} & \underline{77.7}       \\
      NOSMOG \citep{tian2022learning}    & 68.1{\small ±3.0}             & 67.1{\small ±2.1}             & 77.4{\small ±0.8}             & 79.8{\small ±1.5}             & 87.1{\small ±1.5}             & 92.6{\small ±0.7}             & 95.5{\small ±0.1}             & 70.4{\small ±1.2}             & 45.3{\small ±0.7}             & 63.5{\small ±0.8}             & 74.7                   \\ \midrule\midrule
      \rowcolor{Gray} \textbf{SimMLP}    & \textbf{80.5{\small ±2.2}}    & \textbf{72.8{\small ±1.6}}    & \textbf{86.4{\small ±0.5}}    & \underline{87.6{\small ±1.0}} & \textbf{93.9{\small ±0.6}}    & \textbf{94.5{\small ±0.2}}    & \textbf{96.0{\small ±0.2}}    & \textbf{78.5{\small ±1.5}}    & \textbf{49.4{\small ±0.5}}    & \textbf{66.1{\small ±1.1}}    & \textbf{80.6}          \\
      \bottomrule
    \end{tabular}
  }
\end{table*}

\subsection{Full Inductive Bias Analysis}
\label{sec:ind bias full}

See Table \ref{tab:inductive bias full}

\begin{table*}[!t]
  \centering
  \caption{SimMLP shares two inductive biases with GNNs, i.e., homophily and local structure importance, which are measured by \textbf{smoothness} and \textbf{min-cut}, respectively.
  }
  \label{tab:inductive bias full}

  \resizebox{\linewidth}{!}{
    \begin{tabular}{l  cccccc cccccc}
      \toprule
      ~                                  & \multicolumn{6}{c}{\textbf{Homophily} (Smoothness$\downarrow$)} & \multicolumn{6}{c}{Local Structure Importance (Min-Cut$\uparrow$)}                                                                                                                                                                                           \\ \cmidrule(lr){2-7}\cmidrule(lr){8-13}
      Methods                            & Cora                                                            & Citeseer                                                           & Pubmed         & Computer       & Photo          & \textbf{\textit{Avg.}} & Cora           & Citeseer       & Pubmed         & Computer       & Photo          & \textbf{\textit{Avg.}} \\ \midrule\midrule
      Raw Node Feature                   & 0.822                                                           & 0.783                                                              & 0.734          & 0.539          & 0.540          & 0.684                  & $-$            & $-$            & $-$            & $-$            & $-$            & $-$                    \\ \midrule\midrule
      SAGE \citep{hamilton2017inductive} & 0.113                                                           & 0.184                                                              & 0.143          & 0.156          & 0.109          & 0.141                  & 0.924          & 0.943          & 0.918          & 0.854          & 0.872          & 0.902                  \\
      BGRL \citep{thakoor2022largescale} & 0.155                                                           & 0.102                                                              & 0.333          & 0.251          & 0.203          & 0.209                  & 0.885          & 0.935          & 0.856          & 0.834          & 0.849          & 0.872                  \\ \midrule\midrule
      MLP \citep{zhang2022graphless}     & 0.463                                                           & 0.444                                                              & 0.485          & 0.456          & 0.432          & 0.456                  & 0.666          & 0.804          & 0.863          & 0.718          & 0.747          & 0.759                  \\
      GLNN \citep{zhang2022graphless}    & 0.282                                                           & 0.268                                                              & 0.421          & 0.355          & 0.398          & 0.345                  & 0.886          & 0.916          & 0.793          & 0.804          & 0.811          & 0.842                  \\
      NOSMOG \citep{tian2022learning}    & 0.267                                                           & 0.230                                                              & 0.394          & 0.306          & \textbf{0.277} & 0.295                  & 0.902          & 0.932          & 0.834          & 0.838          & 0.823          & 0.866                  \\
      VQGraph \citep{yang2023vqgraph}    & 0.253                                                           & 0.212                                                              & 0.396          & 0.328          & 0.310          & 0.300                  & 0.914          & 0.940          & 0.831          & 0.858          & 0.836          & 0.876                  \\  \midrule\midrule
      \rowcolor{Gray}\textbf{SimMLP}     & \textbf{0.196}                                                  & \textbf{0.170}                                                     & \textbf{0.360} & \textbf{0.299} & 0.288          & \textbf{0.263}         & \textbf{0.934} & \textbf{0.958} & \textbf{0.886} & \textbf{0.901} & \textbf{0.860} & \textbf{0.908}         \\ \bottomrule
    \end{tabular}
  }
\end{table*}

\section{Training Efficiency}
\label{sec:efficiency}

Table \ref{tab:runing time and memory} presents a comparison of the running time and memory usage between SimMLP and other baselines, namely GAT \citep{velickovic2018graph}, GRACE \citep{zhu2020deep}, and BGRL \citep{thakoor2022largescale}. Apart from the significant inference acceleration, SimMLP has less training time and memory usage. In particular, GAT with 4 attention heads imposes a substantial computational consumption in model training. This is highly probable to be the consumption in learning attention scores. GRACE utilizes InfoNCE loss to align the consistency between two graph views, where the similarity measurements might lead to significant computational overhead. Compared to this method, SimMLP demonstrates improvements in terms of memory usage ($3.8\sim6.8\times$) and training time ($4.8\sim8.3\times$). BGRL employs bootstrap \citep{grill2020bootstrap} to alleviate the need for negative samples in InfoNCE, thus alleviating significant computational usage in measuring the distance between negative pairs. However, SimMLP remains more efficient than BGRL due to the use of MLP encoder.

\begin{table*}[!h]
  \centering
  \caption{Computational requirements of different baseline methods on a set of standard benchmark graphs. The experiments are performed on a 24GB Nvidia GeForce RTX 3090.}
  \resizebox{\linewidth}{!}{
    \begin{tabular}{lcccccccccc}
      \toprule
                                     & \multicolumn{2}{c}{Computer} & \multicolumn{2}{c}{Photo} & \multicolumn{2}{c}{Coauthor-CS} & \multicolumn{2}{c}{Coauthor-Phys} & \multicolumn{2}{c}{Wiki-CS}                                                                                                    \\ \cmidrule(lr){2-3}\cmidrule(lr){4-5}\cmidrule(lr){6-7}\cmidrule(lr){8-9}\cmidrule(lr){10-11}
      Methods                        & Memory                       & Training Time             & Memory                          & Training Time                     & Memory                      & Training Time     & Memory           & Training Time      & Memory           & Training Time     \\ \midrule\midrule
      GAT                            & 5239 MB                      & 73.8 (s)                  & 2571 MB                         & 41.9 (s)                          & 2539 MB                     & 60.4 (s)          & 13199 MB         & 265.2 (s)          & 4568 MB          & 74.4 (s)          \\
      GRACE                          & 8142 MB                      & 349.5 (s)                 & 2755 MB                         & 138.4 (s)                         & 11643 MB                    & 261.4 (s)         & 16294 MB         & 573.2 (s)          & 5966 MB          & 290.9 (s)         \\
      BGRL                           & 2196 MB                      & 96.8 (s)                  & 1088 MB                         & 64.1 (s)                          & 2513 MB                     & 129.9 (s)         & 5556 MB          & 273.8 (s)          & 1899 MB          & 108.8 (s)         \\
      \rowcolor{Gray}\textbf{SimMLP} & \textbf{1969 MB}             & \textbf{53.4 (s)}         & \textbf{694 MB}                 & \textbf{27.0 (s)}                 & \textbf{1716 MB}            & \textbf{54.8 (s)} & \textbf{3920 MB} & \textbf{110.7 (s)} & \textbf{1590 MB} & \textbf{35.5 (s)} \\ \bottomrule
    \end{tabular}
  }
  \label{tab:runing time and memory}
\end{table*}

Our SimMLP can scale to very large graphs via mini-batch training. We report the time and memory consumption on OGB-Product dataset in Table \ref{tab:complexity ogb-product}. We compare our SimMLP against BGRL, an efficient SSL method on graphs. The results demonstrate the scalability and practicality of SimMLP in handling real-world applications.

\begin{table}[!h]
  \centering
  \caption{The time and memory consumption on large-scaled OGB-product dataset.}
  \label{tab:complexity ogb-product}
  \resizebox{\linewidth}{!}{
    \begin{tabular}{lcc}
      \toprule
      \textbf{Method}                & \textbf{Training Time (Per Epoch)} & \textbf{Memory Consumption} \\ \midrule
      BGRL                           & 263s                               & 17,394MB                    \\
      \rowcolor{Gray}\textbf{SimMLP} & \textbf{158s}                      & \textbf{11,993MB}           \\
      \bottomrule
    \end{tabular}
  }
\end{table}

\begin{table}[!h]
  \centering
  \caption{Comparison between MLP-based methods in training the MLP for downstream node classification (5000 epochs). }
  \label{tab:time mlp}
  \resizebox{\linewidth}{!}{
    \begin{tabular}{lccccc}
      \toprule
                                      & Cora & Citeseer & Pubmed & Flickr & Arxiv \\ \midrule\midrule
      GLNN                            & 1.6s & 1.9s     & 2.0s   & 2.5s   & 3.3s  \\
      VQGraph                         & 1.9s & 2.3s     & 2.7s   & 3.2s   & 4.5s  \\
      NOSMOG                          & 2.3s & 2.5s     & 2.7s   & 3.6s   & 4.7s  \\ \midrule\midrule
      \rowcolor{Gray} \textbf{SimMLP} & 1.6s & 1.9s     & 1.9s   & 2.5s   & 3.2s  \\
      \bottomrule
    \end{tabular}
  }
\end{table}

\section{Comprehensive Ablation Study}
\label{sec:ablation}

\subsection{The necessity in incorporating structural information}

The use of GNN encoder in learning structure-aware knowledge is essential in SimMLP, as it directly aligns the embeddings of two encoders. Without the GNN encoder, the model will fail to capture the fine-grained and generalizable correlation between node features and graph structures, demonstrated in Table \ref{tab:ablation gnn}.

\begin{table*}[!h]
  \centering
  \caption{The ablation study on incorporating structural information using GNNs. Without the GNN encoder (i.e., only using MLPs), the model performance will be significantly decreased.}
  \resizebox{\linewidth}{!}{
    \begin{tabular}{lcccccccccc}
      \toprule
      Methods          & Cora                        & Citeseer                   & Pubmed                      & Computer                    & Photo                       & Co-CS                       & Co-Phys                     & Wiki-CS                     & Flickr                      & Arxiv                       \\ \midrule\midrule
      \textbf{SimMLP}  & \textbf{84.60{\tiny ±0.24}} & \textbf{73.52{\tiny±0.53}} & \textbf{86.99{\tiny ±0.09}} & \textbf{88.46{\tiny ±0.16}} & \textbf{94.28{\tiny ±0.08}} & \textbf{94.87{\tiny ±0.07}} & \textbf{96.17{\tiny ±0.03}} & \textbf{81.21{\tiny ±0.13}} & \textbf{49.85{\tiny ±0.09}} & \textbf{71.12{\tiny ±0.10}} \\
      \textbf{w/o GNN} & 55.91{\tiny ±0.66}          & 57.36{\tiny ±0.33}         & 79.93{\tiny ±0.32}          & 72.76{\tiny ±0.71}          & 77.05{\tiny ±0.18}          & 91.19{\tiny ±0.13}          & 93.35{\tiny ±0.12}          & 73.87{\tiny ±0.26}          & 45.82{\tiny ±0.07}          & 54.83{\tiny ±0.41}          \\ \bottomrule
    \end{tabular}
  }
  \label{tab:ablation gnn}
\end{table*}


\subsection{The design choice of Strategy 1}
\label{sec:aggregation type}

In this section, we analyze some design choices of Strategy 1, i.e., using MLP to approximate GNN. The learning process is similar to SGC \citep{wu2019simplifying} or APPNP \citep{gasteiger2018combining} that decompose feature transformation and message passing. In SimMLP, we consider using normalized Laplacian matrix to direct the message passing due to its simplicity. Based on the normalization, we have three design choices. We dub them as (1) \textbf{Col}: using column-normalized Laplacian matrix $\tilde{\mD}^{-1}\tilde{\mA}$ for message passing, (2) \textbf{Row}: using row-normalized Laplacian matrix $\tilde{\mA}\tilde{\mD}^{-1}$ for message passing, and (3) \textbf{Bi}: using bi-normalized Laplacian matrix $\tilde{\mD}^{-1/2}\tilde{\mA}\tilde{\mD}^{-1/2}$ for message passing. Here $\tilde{\mA} = \mA + \mI$ and $\tilde{\mD}$ is the diagonal matrix of node degrees of $\tilde{\mA}$. We present the results of these three choices on ten benchmark datasets, shown in Table \ref{sec:aggregation type}.

\begin{table*}[!h]
  \centering
  \caption{Ablation study on node aggregation choices. {\tt Col} indicates column-normalized Laplacian aggregation matrix $\tilde{\mD}^{-1}\tilde{\mA}$, {\tt Row} indicates row-normalized Laplacian aggregation matrix $\tilde{\mA}\tilde{\mD}^{-1}$, and {\tt Bi.} indicates bi-normalized Laplacian aggregation matrix $\tilde{\mD}^{-1/2}\tilde{\mA}\tilde{\mD}^{-1/2}$. SimMLP employs {\tt Bi.} since it consistently outperforms others even though the improvements may not be significant. }
  \resizebox{\linewidth}{!}{
    \begin{tabular}{lcccccccccc}
      \toprule
                      & Cora                        & Citeseer                    & Pubmed                      & Computer                    & Photo                       & Co-CS                       & Co-Phys                     & Wiki-CS                     & Flickr                      & Arxiv                       \\ \midrule\midrule
      \textbf{{ Bi.}} & \textbf{84.60{\tiny ±0.24}} & \textbf{73.52{\tiny ±0.53}} & \textbf{86.99{\tiny ±0.09}} & \textbf{88.46{\tiny ±0.16}} & \textbf{94.28{\tiny ±0.08}} & \textbf{94.87{\tiny ±0.07}} & \textbf{96.17{\tiny ±0.03}} & \textbf{81.21{\tiny ±0.13}} & \textbf{49.85{\tiny ±0.09}} & \textbf{71.12{\tiny ±0.10}} \\
      { Col}          & 84.14{\tiny ±0.34}          & 73.48{\tiny ±0.53}          & 86.92{\tiny ±0.08}          & 87.93{\tiny ±0.27}          & 93.11{\tiny ±0.15}          & 94.81{\tiny ±0.06}          & 96.09{\tiny ±0.03}          & 80.62{\tiny ±0.30}          & 49.15{\tiny ±0.16}          & 71.03{\tiny ±0.09}          \\
      { Row}          & 84.09{\tiny ±0.32}          & 73.49{\tiny ±0.54}          & 86.92{\tiny ±0.08}          & 87.96{\tiny ±0.27}          & 93.07{\tiny ±0.15}          & 94.82{\tiny ±0.06}          & 96.07{\tiny ±0.04}          & 80.63{\tiny ±0.25}          & 49.18{\tiny ±0.10}          & 71.04{\tiny ±0.09}          \\ \bottomrule
    \end{tabular}
  }
  \label{tab:aggregation type}
\end{table*}

In this table, we observe that there is no significant difference in performance among the various aggregation methods. All of these methods can achieve desirable performance. Nevertheless, the bi-normalized aggregation (\textbf{Bi.}) consistently outperforms the others. Actually, we can directly use the message passing functions of SGC or APPNP. For SGC, we do not observe significant performance differences compared to the discussed three choices. For APPNP that performs message passing based on page-rank, we consider obtaining the page-rank aggregation matrix would lead to significant time consumption, especially with graph structural augmentations. We leave this in the future work.

\begin{figure*}[!ht]
  \centering
  \subfloat{\includegraphics[width=0.20\linewidth]{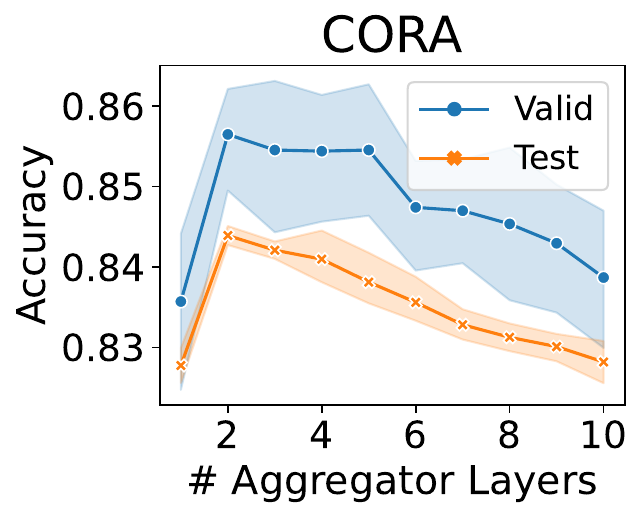}}
  \subfloat{\includegraphics[width=0.20\linewidth]{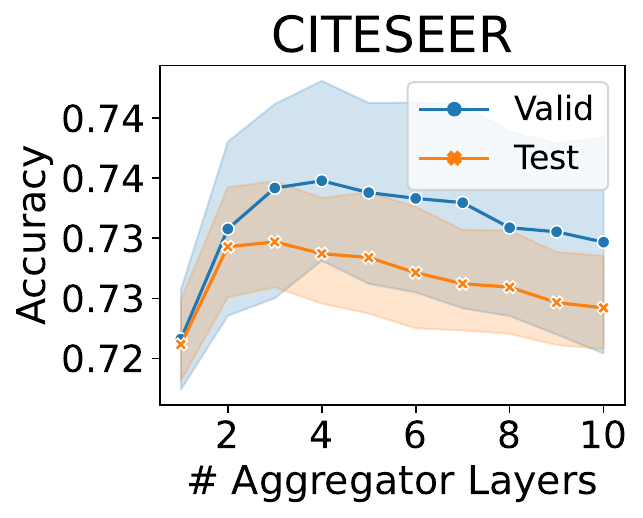}}
  \subfloat{\includegraphics[width=0.20\linewidth]{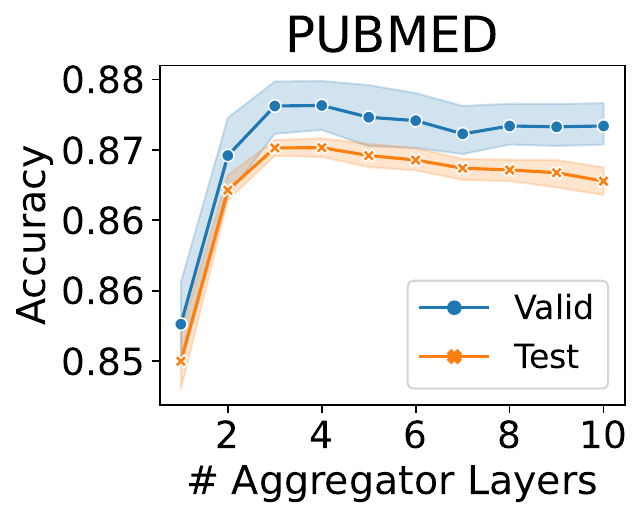}}
  \subfloat{\includegraphics[width=0.20\linewidth]{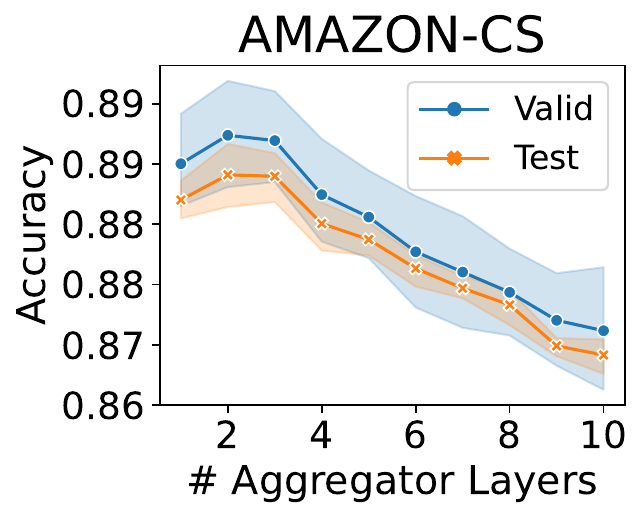}}
  \subfloat{\includegraphics[width=0.20\linewidth]{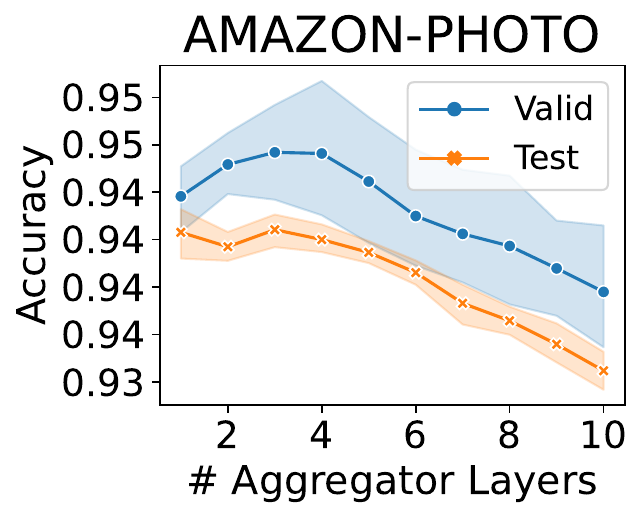}}
  \caption{Node classification accuracy on transductive setting with different aggregation layers. }
  \label{fig:aggregation layer}
\end{figure*}

Apart from the choice of message passing methods, determining the message passing layers is also important. We show the performance of SimMLP with varying numbers of message passing layers on five benchmark datasets in Figure \ref{fig:aggregation layer}. We observe the optimal performance is achieved with 2 or 3 layers, which is consistent with prior research on GNNs \citep{li2019deepgcns}. It might be because a high number of message passing layers can result in over-smoothing.

\subsection{How does Strategy 2 (augmentation) prevent trivial solutions?}
\label{sec:augment}

Additionally, we conduct a detailed analysis of how augmentations impact model performance. Figure \ref{fig:augment} illustrates the model performance at different augmentation probabilities on Cora, Citeseer, Pubmed, Computer, and Photo datasets under the transductive setting. The augmentation ratio is searched among \{0.0, 0.25, 0.5, 0.75\}. These figures enable us to gain insight into the specific effects of augmentations on model performance.

\begin{figure*}[!h]
  \centering
  \subfloat{\includegraphics[width=0.20\linewidth]{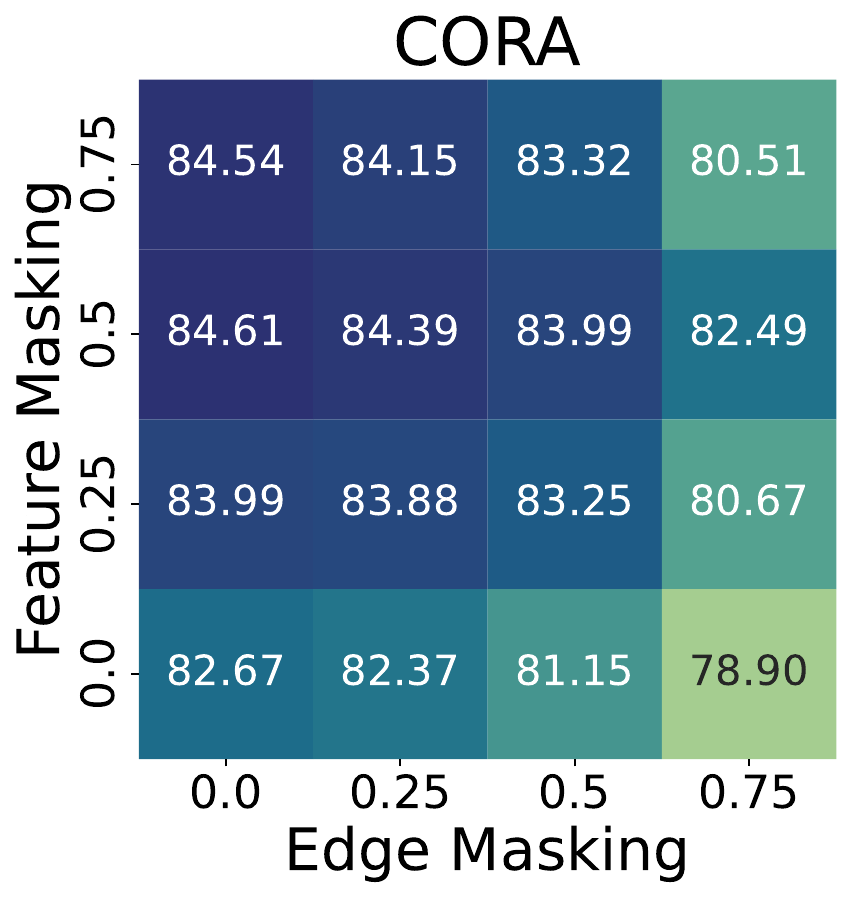}}
  \subfloat{\includegraphics[width=0.20\linewidth]{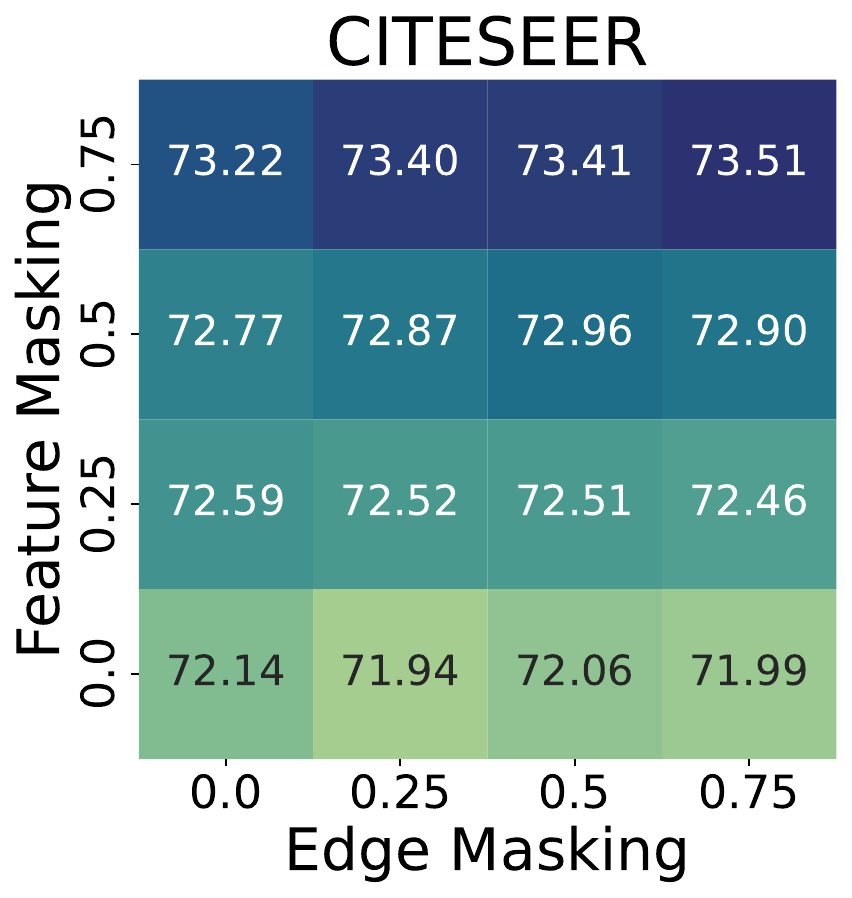}}
  \subfloat{\includegraphics[width=0.20\linewidth]{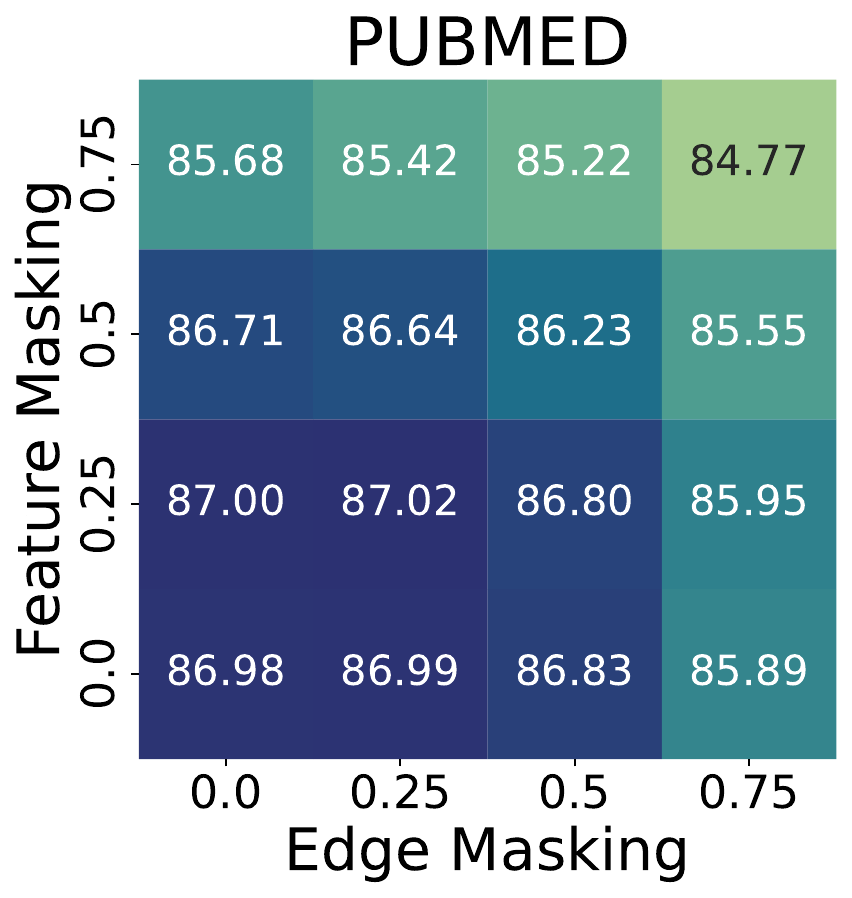}}
  \subfloat{\includegraphics[width=0.20\linewidth]{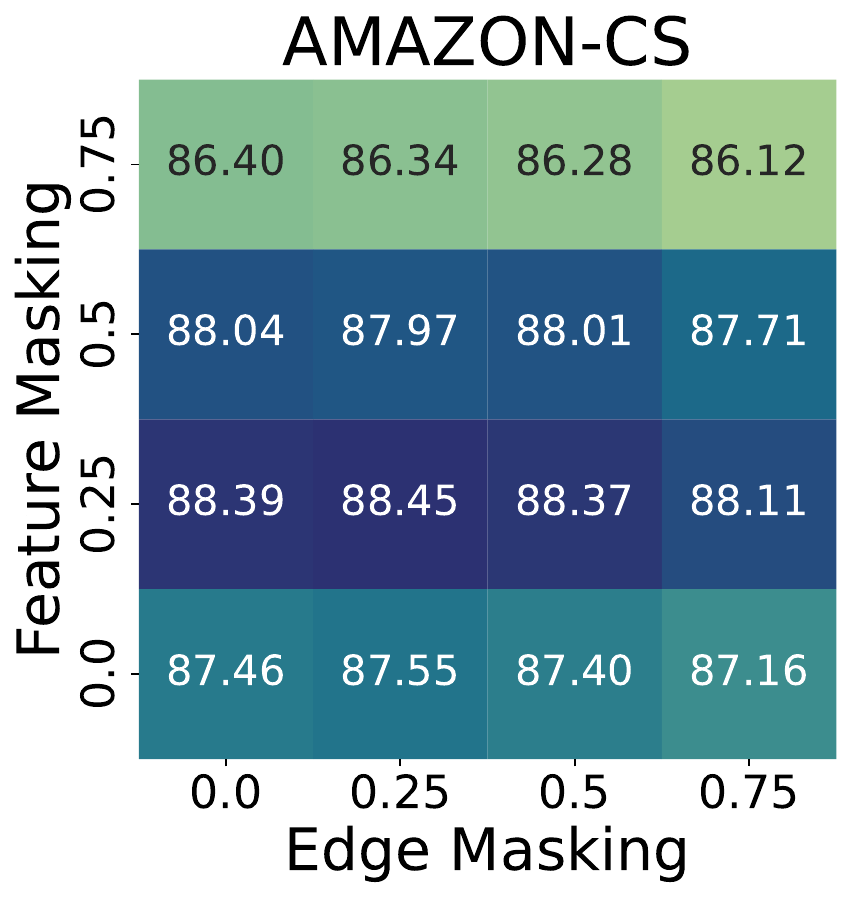}}
  \subfloat{\includegraphics[width=0.20\linewidth]{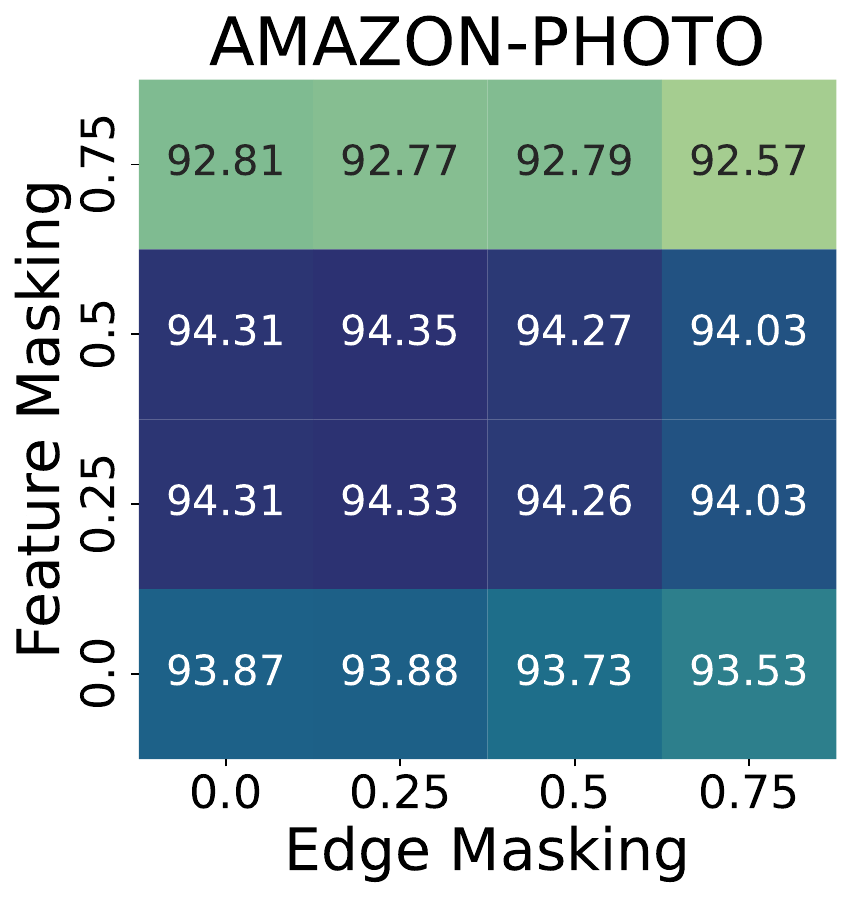}}
  \caption{Node classification accuracy on transductive setting with different augmentation ratios. }
  \label{fig:augment}
\end{figure*}

\subsection{How the reconstruction term in Equation \ref{eq:objective} works?}
\label{sec:recon term}

In this section, we evaluate the role of the reconstruction term of SimMLP in Equation \ref{eq:objective}. We treat the term serves as a regularizer that mitigates the potential distribution shifts. It works like positional embedding \citep{dwivedi2022graph} that preserves more localized information on GNN embeddings. We show the impact of the reconstruction term on model performance in Table \ref{tab:recon}. Our observations indicate the reconstruction term might be important in large-scale datasets, e.g., Arxiv. It might be because these datasets contain more noise.

\begin{table*}[!ht]
  \centering
  \caption{Reconstruction term in Equation \ref{eq:objective} serves as a regularizer, preventing the potential distribution shifts.}
  \resizebox{\linewidth}{!}{
    \begin{tabular}{lcccccccccc}
      \toprule
      Methods         & Cora                        & Citeseer                   & Pubmed                      & Computer                    & Photo                       & Co-CS                       & Co-Phys                     & Wiki-CS                     & Flickr                      & Arxiv                       \\ \midrule
      \textbf{SimMLP} & \textbf{84.60{\tiny ±0.24}} & \textbf{73.52{\tiny±0.53}} & \textbf{86.99{\tiny ±0.09}} & \textbf{88.46{\tiny ±0.16}} & \textbf{94.28{\tiny ±0.08}} & \textbf{94.87{\tiny ±0.07}} & \textbf{96.17{\tiny ±0.03}} & \textbf{81.21{\tiny ±0.13}} & \textbf{49.85{\tiny ±0.09}} & \textbf{71.12{\tiny ±0.10}} \\
      {w/o Rec.}      & 84.37{\tiny ±0.27}          & 73.18{\tiny ±0.24}         & 86.86{\tiny ±0.10}          & 88.25{\tiny ±0.07}          & 94.15{\tiny ±0.07}          & 94.64{\tiny ±0.06}          & 96.01{\tiny ±0.07}          & 81.10{\tiny ±0.13}          & 49.60{\tiny ±0.11}          & 70.38{\tiny ±0.22}          \\ \bottomrule
    \end{tabular}
  }
  \label{tab:recon}
\end{table*}

\subsection{Different GNN Architectures}

It is straightforward to adapt our approach from GCN to other GNN architectures. In Table \ref{tab:gnn backbone}, we show results of GCN, SAGE, and APPNP on Cora, Citeseer, and Pubmed. SimMLP consistently enhances performance across GNN architectures, owing to the capability of SSL to capture generalizable patterns. This underscores SimMLP's adaptability and effectiveness across a wide range of tasks and architectures.

\begin{table*}[!h]
  \centering
  \caption{Model performance with different GNN backbones.}
  \label{tab:gnn backbone}
  \resizebox{0.9\linewidth}{!}{
    \begin{tabular}{lcccccc}
      \toprule
               & GCN  & \textbf{GCN + SimMLP} & SAGE & \textbf{SAGE + SimMLP} & APPNP & \textbf{APPNP + SimMLP} \\ \midrule
      Cora     & 82.1 & \textbf{84.6}         & 81.4 & \textbf{84.1}          & 81.4  & \textbf{84.3}           \\
      Citeseer & 70.7 & \textbf{73.5}         & 70.4 & \textbf{73.5}          & 70.3  & \textbf{73.6}           \\
      Pubmed   & 85.6 & \textbf{87.0}         & 85.9 & \textbf{86.9}          & 85.7  & \textbf{86.8}           \\
      \bottomrule
    \end{tabular}
  }
\end{table*}

\end{document}